%% file: arxiv.tex
\documentclass[letterpaper]{article}

\PassOptionsToPackage{numbers, compress}{natbib}
\usepackage{fullpage}
\usepackage[utf8]{inputenc}
\usepackage{amssymb,amsmath}
\usepackage{algorithm,algorithmic}

\DeclareMathOperator*{\argmax}{arg\,max}
\DeclareMathOperator*{\argmin}{arg\,min}
\newcommand{\shortrnote}[1]{ &  &  & \text{\footnotesize\llap{#1}}}
\newcommand{\longrnote}[1]{ &  &  \\   &  &  &  &  & \notag \text{\footnotesize\llap{#1}}}
\usepackage{multicol}
\usepackage{pifont}
\usepackage{times}
\usepackage{amsmath}
\usepackage{amsfonts}
\usepackage{amsthm}
\usepackage{amssymb}
\usepackage{algorithm}
\usepackage{algorithmic}

\usepackage{graphicx, xcolor}
\usepackage{caption}
\usepackage{subcaption}
\usepackage{tabu}
\usepackage{float}
\usepackage{algorithm}
\usepackage{algorithmic}
\usepackage{float}
\newfloat{algorithm}{t}{lop}

\usepackage{epstopdf}
\usepackage{cuted}
\usepackage{enumerate}
\usepackage{enumitem}
\usepackage{verbatim}
\usepackage{hyperref}
\usepackage{multirow}
\usepackage{hhline}
\usepackage{url}
\usepackage[utf8]{inputenc}
\usepackage{pifont}
\usepackage[english]{babel} 
\usepackage{supertabular}
\usepackage{natbib}
\usepackage{float}
\def\BState{\State\hskip-\ALG@thistlm}

\usepackage[utf8]{inputenc} 
\usepackage[T1]{fontenc}    
\usepackage{url}            
\usepackage{booktabs}       
\usepackage{nicefrac}       
\usepackage{microtype}      
\usepackage{bm}

\newcommand{\Ocal}{{\mathcal{O}}}
\newcommand{\R}{{\mathbb{R}}}
\def\ba{{\boldsymbol{a}}}
\def\bb{{\boldsymbol{b}}}
\def\be{{\boldsymbol{e}}}
\def\bg{{\boldsymbol{g}}}

\def\bx{{\boldsymbol{x}}}
\def\bw{{\boldsymbol{w}}}
\def\by{{\boldsymbol{y}}}
\def\bv{{\boldsymbol{v}}}

\def\bz{{\boldsymbol{z}}}
\def\br{{\boldsymbol{r}}}
\def\Acal{{\mathcal{A}}}
\usepackage{comment} 
\usepackage{footnote} 
\newcommand{\La}{{\mathcal{L}}}
\newtheorem{assumption}{Assumption}[section]
\newtheorem{theorem}{Theorem}[section]
\newtheorem{corollary}{Corollary}[theorem]
\newtheorem{lemma}[theorem]{Lemma}
\newtheorem{definition}[theorem]{Definition}

\newtheorem{claim}[theorem]{Claim}

\newcommand{\A}{\mathcal{A}}

\newcommand{\define}{\overset{\text{def}}{=}}

\usepackage{authblk}

\title{Primal-Dual Block Frank-Wolfe}
\author[$\dagger$]{Qi Lei\footnote{Both authors contribute equally.}}
\author[$\ast\dagger$]{\ \ Jiacheng Zhuo}
\author[$\dagger$]{\ \ Constantine Caramanis}
\author[$\dagger\ddagger$]{\ \ Inderjit S. Dhillon}
\author[$\dagger$]{\ \ Alexandros G. Dimakis}
\affil[ ]{$^\dagger$ UT Austin \ \  $^\ddagger$ Amazon}
\affil[ ]{\texttt{\{leiqi@oden., jzhuo@, constantine@, inderjit@cs., dimakis@austin.\}utexas.edu} }

\begin{document}

\maketitle

\begin{abstract}


    We propose a variant of the Frank-Wolfe algorithm for solving a class of sparse/low-rank optimization problems. Our formulation includes Elastic Net, regularized SVMs and phase retrieval as special cases. The proposed Primal-Dual Block Frank-Wolfe algorithm reduces the per-iteration cost while maintaining linear convergence rate.
    The per iteration cost of our method depends on the structural complexity of the solution (i.e. sparsity/low-rank) instead of the ambient dimension.
    We empirically show that our algorithm outperforms the state-of-the-art methods on (multi-class) classification tasks.
\end{abstract}

\section{Introduction}
We consider optimization problems of the form:
$$
\min_{\bm{x} \in C}: \,\, \sum_i f_i(\bm{a}_i^{\top}\bm{x}) + g(\bm{x}),
$$
directly motivated by regularized and constrained Empirical Risk Minimization (ERM). Particularly, we are interested in problems whose solution has special ``simple'' structure like low-rank or sparsity. 
The sparsity constraint applies to large-scale multiclass/multi-label classification, low-degree polynomial data mapping \cite{chang2010training}, random feature kernel machines \cite{rahimi2008random}, and Elastic Net \cite{zou2005regularization}. 
Motivated by recent applications in low-rank multi-class SVM, phase retrieval, matrix completion, affine rank minimization and other problems (e.g., \cite{dudik2012lifted,pong2010trace,argyriou2008convex,candes2015phase}), we also consider settings where the constraint $\bm{x} \in C$ (e.g., trace norm ball) while convex, may be difficult to project onto. A wish-list for this class of problems would include an algorithm that (1) exploits the function finite-sum form and the simple structure of the solution, (2) achieves linear convergence for smooth and strongly convex problems, (3) 
does not pay a heavy price for the projection step. 
    
We propose a Frank-Wolfe (FW) type method that attains these three goals.  This does not come without challenges:
Although it is currently well-appreciated that FW type algorithms avoid the cost of projection \cite{jaggi2013revisiting,allen2017linear}, the benefits are limited to constraints that are hard to project onto, like the trace norm ball.
For problems like phase retrieval and ERM for multi-label multi-class classification, the gradient computation requires large matrix multiplications. This dominates the per-iteration cost, and the existing FW type methods do not asymptotically reduce time complexity per iteration, even without paying the expensive projection step. 
Meanwhile, for simpler constraints like the $\ell_1$ norm ball or the simplex, it is unclear if FW can offer any benefits compared to other methods.
Moreover, as is generally known, FW suffers from sub-linear convergence rate even for well-conditioned problems that enjoy strong convexity and smoothness. 

In this paper we tackle the challenges by exploiting the special structure induced by the constraints and FW steps. 
%
We propose a variant of FW that we call Primal-Dual Block Frank Wolfe. The main advantage is that the computational complexity depends only on the sparsity of the solution, rather than the ambient dimension, i.e. it is \textit{dimension free}. This is achieved by conducting \textit{partial updates} in each iteration, i.e., sparse updates for $\ell_1$ and low-rank updates for the trace norm ball. 
While the benefits of \textit{partial updates} seem unclear for the original problem, they significantly benefit a primal-dual reformulation.  
This reduces the per iteration cost to roughly a ratio of $\frac{s}{d}$ compared to naive Frank-Wolfe, where $s$ is the sparsity (or rank) of the optimal solution, and $d$ is the feature dimension. Meanwhile, the per iteration progress of our proposal is comparable to a full gradient descent step, thus retaining linear convergence rate. 

For strongly convex and smooth $f$ and $g$ we show that 
our algorithm achieves linear convergence with per-iteration cost $sn$ over $\ell_1$-norm ball, where $s$ upper bounds the sparsity of the primal optimal. Specifically, for sparse ERM with smooth hinge loss or quadratic loss with $\ell_2$ regularizer, our algorithm yields an overall $\Ocal(s(n+\kappa)\log\frac{1}{\epsilon})$ time complexity to reach $\epsilon$ duality gap, where $\kappa$ is the condition number (smoothness divided by strong convexity). Our theory has minimal requirements on the data matrix $A$. 

Experimentally we observe our method yields significantly better performance compared to prior work, especially when the data dimension is large and the solution is sparse.

\section{Related Work}

We review relevant algorithms that improve the overall performance of Frank-Wolfe type methods. Such improvements are roughly obtained for two reasons: the enhancement on convergence speed and the reduction on iteration cost. Very few prior works benefit in both.  

Nesterov's acceleration has proven effective 
as in Stochastic Condition Gradient Sliding (SCGS) \cite{lan2016conditional} and other variants \cite{weintraub1985accelerating,meyer1974accelerated,garber2015faster}. Restarting techniques dynamically adapt to the function geometric properties and fills in the gap between sublinear and linear convergence for FW method \cite{kerdreux2018restarting}. Some variance reduced algorithms obtain linear convergence as in \cite{hazan2016variance}, however, the number of inner loops grows
significantly and hence the method is not computationally efficient.

Linear convergence has been obtained specifically for polytope constraints like \cite{nanculef2014novel,kumar2011linearly}, as well as the work proposed in \cite{lacoste2015global,goldfarb2017linear} that use the Away-step Frank Wolfe and Pair-wise Frank Wolfe, and their stochastic variants. 
One recent work \cite{allen2017linear} focuses on trace norm constraints and proposes a FW-type algorithm that yields similar progress as projected gradient descent per iteration but is almost projection free. However, in many applications where gradient computation dominates the iteration complexity, the reduction on projection step doesn't necessarily produce asymptotically better iteration costs.

The sparse update introduced by FW steps was also appreciated by \cite{lacoste2012block}, where they conducted dual updates with a focus on SVM with polytope constraint. Their algorithm yields low iteration costs but still suffer from sub-linear convergence.



On the other hand, the recently popularized primal-dual formulation $\min_{\bx}\max_{\by}\{g(\bx)+\by^\top A\bx - f(\by)\}$ has proven useful for different machine learning tasks like reinforcement learning, ERM, and robust optimization \cite{du2018linear}. Especially for the ERM related problems, the primal-dual formulation still inherits the finite-sum structure from the primal form, and could be used to reduce variance \cite{zhang2014stochastic,wang2017exploiting} or reduces communication complexity in the distributed setting \cite{xiao2019dscovr}. However, there has been a lack of studies for constrained problems for this form. Meanwhile, as we emphasize in this paper, since the most time-consuming part of this formulation comes from the bilinear term $\by^\top A\bx$, we are able to exploit simple structure from the optimal solutions.


\section{Setup}
\subsection{Notation}
We briefly introduce the notation used throughout the paper. We use bold lower case letter to denote vectors, capital letter to represent matrices. $\|\cdot\|$ is $\ell_2$ norm for vectors and Frobenius norm for matrices unless specified otherwise. $\|\cdot\|_*$ indicates the trace norm for a matrix.

We say a function $f$ is $\alpha$ strongly convex if 
    $f(\by) \geq f(\bx) + \langle \bg, \by-\bx\rangle +\frac{\alpha}{2}\|\by-\bx\|^2,$
where $\bg\in \partial f(\bx)$ is any sub-gradient of $f$.  
Similarly, $f$ is $\beta$-smooth when
$    f(\by) \leq f(\bx) + \langle \bg, \by-\bx\rangle +\frac{\beta}{2}\|\by-\bx\|^2.$
We use $f^*$ to denote the convex conjugate of $f$, i.e., $f^*(\by)\define \max_{\bx}\langle \bx, \by\rangle - f(\bx)$. Some more parameters are problem-specific and are defined when needed.
\subsection{A Theoretical Vignette}
\label{sec:vignette}
To elaborate the techniques we use to obtain the linear convergence for our Frank-Wolfe type algorithm,
we consider the $\ell_1$ norm constrained problem as an illustrating example:
\begin{equation}
\label{eqn:l1_example}
    \argmin_{\bx\in \R^d, \|\bx\|_1\leq \tau} f(\bx),
\end{equation}
where $f$ is $L$-smooth and $\mu$-strongly convex.
If we invoke the Frank Wolfe algorithm, we compute
\begin{equation} 
\bx^{(t)}\leftarrow (1-\eta) \bx^{(t-1)} + \eta \tilde{\bx}, \quad
\text{ where } 
\tilde{\bx} \leftarrow \argmin_{\|\bx\|_1\leq \tau} \langle \nabla f(\bx^{(t-1)}), \bx\rangle.
\label{eqn:old_fw}
\end{equation}
Even when the function $f$ is smooth and strongly convex, 
\eqref{eqn:old_fw} converges sublinearly.
As inspired by \cite{allen2017linear}, 
if we assume the optimal solution is $s$-sparse, 
we can enforce a sparse update while maintaining linear convergence by a mild modification on \eqref{eqn:old_fw}: 
\begin{equation}
\bx^{(t)}\leftarrow (1-\eta) \bx^{(t-1)} + \eta \tilde{\bx}, 
\text{ where } 
 \tilde{\bx} \leftarrow \argmin_{\|\bx\|_1\leq \tau, \|\bx\|_0\leq s}  \{\langle \nabla f(\bx^{(t-1)}),  \bx\rangle +\frac{L}{2} \eta \|\bx^{(t-1)} - \bx\|_2^2\}.
    \label{eqn:block_frank_wolfe}
\end{equation}
We also call this new practice block Frank-Wolfe as in \cite{allen2017linear}. The proof of convergence can be completed within three lines. Let $h_t=f(\bx^{(t)})-f^*$.
\begin{align}
\nonumber 
    h_{t} &= f(\bx^{(t-1)}+\eta (\tilde{\bx} -\bx^{(t-1)}) ) -f^*\\
\nonumber     
    &\leq h_{t-1}+\eta \langle \nabla f(\bx^{(t-1)}), \tilde{ \bx} - \bx^{(t-1)}\rangle +\frac{L}{2} \eta^2 \|\tilde{ \bx}-\bx^{(t-1)}\|^2  \shortrnote{~(Smoothness of $f$)}\\
\nonumber     
    &\leq  h_{t-1}+\eta \langle \nabla f(\bx^{(t-1)}), \bx^* - \bx^{(t-1)}\rangle +\frac{L}{2} \eta^2 \|\bx^*-\bx^{(t-1)}\|^2 \shortrnote{(Definition of $\tilde {\bx}$)} \\
    \label{eqn:fw_progress}
    &\leq (1-\eta+\frac{L}{\mu} \eta^2) h_{t-1} \shortrnote{(by convexity and $\mu$-strong convexity of $f$)}
\end{align}
Therefore, when $\eta=\frac{\mu}{2L}$, $h_{t+1}\leq (1-\frac{\mu}{4L})^t h_1$ and the iteration complexity is $\Ocal(\frac{L}{\mu} \log(1/\epsilon))$ to achieve $\epsilon$ error.

Although we achieve linear convergence, the advantage of overall complexity against classical methods (e.g. Projected Gradient Descend (PGD)) is not shown yet. Luckily, with the update $\tilde{\bx}$ being sparse, it is possible to improve the iteration complexity, while maintaining the linear convergence rate. In order to differentiate, we name the sparse update nature of \eqref{eqn:block_frank_wolfe} as \textit{partial update}.

Consider quadratic functions as an example:$f(\bx)=\frac{1}{2}\bx^\top A\bx$, whose gradient is $A \bx$ for symmetric $A$. As $\tilde{\bx}$ is sparse, One can maintain the value of the gradient efficiently: $A\bx^{(t)}\equiv (1-\eta)A\bx^{(t-1)} +\eta A_{I,:}\tilde{\bx}$, where $I$ is the support set of $\tilde{\bx}$. We therefore reduce the complexity of one iteration to $\Ocal(sd)$, compared to $\Ocal(d^2)$ with PGD. Similar benefits hold when we replace $\bx$ by a matrix $X$ and conduct a low-rank update on $X$.

The benefit of \textit{partial update} is not limited to quadratic functions.
Next we show that for a class of composite function, we are able to take the full advantage of the \textit{partial update},
by taking a primal-dual re-formulation.

\section{Methodology}

\subsection{Primal-Dual Formulation}
Note that the problem we are tackling is as follows
\begin{equation}
\label{eqn:primal_composite}
    \min_{\bx\in C} \left\{P(\bx)\equiv  \frac{1}{n}\sum_{i=1}^nf_i(\ba_i^\top\bx) + g(\bx) \right\},
\end{equation}
We first focus on the setting where $\bx \in \mathbb{R}^d$ is a vector and $C$ is the $\ell_1$-norm ball. 
This form covers general classification or regression tasks with $f_i$ being some loss function and $g$ being a regularizer. 
Extension to matrix optimization over a trace norm ball is introduced in Section \ref{sec:trace_norm}.

Even with the constraint, we could reform \eqref{eqn:primal_composite} as a primal-dual convex-concave saddle point problem:
\begin{eqnarray}
\label{eqn:primal_dual}
 \eqref{eqn:primal_composite} &\Leftrightarrow& \min_{\bx\in C } \max_{\by }\left\{\La(\bx,\by)\equiv g(\bx) + \frac{1}{n}\langle \by, A\bx\rangle - \frac{1}{n}\sum_{i=1}^n f^*_i(y_i) \right\},
\end{eqnarray}
or its dual formulation:
\begin{eqnarray} 
\label{eqn:dual} 
\eqref{eqn:primal_composite}\Leftrightarrow  \max_{\by }\left\{ D(\by):= \min_{\bx\in C } \left\{g(\bx) + \frac{1}{n}\langle \by, A\bx\rangle\right\} - \frac{1}{n}\sum_{i=1}^n f^*_i(y_i) \right\}. 
\end{eqnarray}
Notice \eqref{eqn:dual} is not guaranteed to have an explicit form.
Therefore some existing FW variants like \cite{lacoste2012block} that optimizes over \eqref{eqn:dual} may not apply.
Instead, we directly solve the convex concave problem \eqref{eqn:primal_dual} and  could therefore solve more general problems, including complicated constraint like trace norm.

Since the computational cost of the gradient $\nabla_{\bx}\La$ and $\nabla_{\by}\La$ is dominated by computing $A^\top\by$ and $A\bx$ respectively, \textit{sparse updates} could reduce computational costs by a ratio of roughly $\Ocal(d/s)$ for updating $\bx$ and $\by$ while achieving good progress.



\subsection{Primal-Dual Block Frank-Wolfe}
With the primal-dual formulation, we are ready to introduce our algorithm. The idea is simple: since the primal variable $\bx$ is constrained over $\ell_1$ norm ball, we conduct block Frank-Wolfe algorithm and achieve an $s$-sparse update. Meanwhile, for the dual variable $\by$ we conduct greedy coordinate ascent method to select and update $k$ coordinates ($k$ determined later). We selected coordinates that allow the largest step, which is usually referred as a Gauss-Southwell rule denoted by {\bf GS-r} \cite{nutini2015coordinate}. Our algorithm is formally presented in Algorithm \ref{alg:PD_l1}. The parameters to set in the algorithm request setting assumptions on $f$ and $g$:
\begin{assumption}
\label{assump:main}
We assume the functions satisfy the following properties:
\begin{itemize}[topsep=0pt,parsep=0pt,partopsep=0pt]
\item Each loss function $f_i$ is convex and $\beta$-smooth, and is $\alpha$ strongly convex over some convex set (could be $\R$), and linear otherwise. 
\item $\max_{i}\|\ba_i\|_2^2\leq R$. Therefore $\frac{1}{n}\sum_{i=1}^nf_i(\ba_i^\top\bx)$ is $\beta R$-smooth. 
\item $g$ is $\mu$-strongly convex and $L$-smooth. 
\end{itemize} 
\end{assumption}
Suitable loss functions $f_i$ include smooth hinge loss \cite{shalev2013stochastic} and quadratic loss function. Relevant applications covered are Support Vector Machine (SVM) with smooth hinge loss, elastic net \cite{zou2005regularization}, and linear regression problem with quadratic loss.





\begin{algorithm*}[thb]
{\footnotesize 
\caption{Frank-Wolfe Block Primal-Dual Method for $\ell_1$ Norm Ball}
\begin{algorithmic}[1]
\STATE {\bfseries Input:}  Training data $A\in \mathbb{R}^{n\times d}$, primal and dual step size $\eta,\delta>0$.
\STATE {\bfseries Initialize:}  $\bx^{(0)}\leftarrow0\in \mathbb{R}^d$, $\by^{(0)}\leftarrow0\in \mathbb{R}^n, \bw^{(0)}\equiv A\bx=0\in \R^n, \bz^{(0)}\equiv A^{\top}\by=0\in\R^d $
\FOR{$t=1,2,\cdots, T$}
\STATE 
Use Block Frank Wolfe to update the primal variable: 
\begin{equation}
 \tilde{ \bx} \leftarrow \argmin_{\|\bx\|_1\leq \lambda, \|\bx\|_0\leq s}  \{\langle \frac{1}{n}\bz^{(t-1)} +\nabla g(\bx^{(t-1)}), \bx\rangle +\frac{L}{2} \eta \|\bx - \bx^{(t-1)}\|^2\}
 \label{eqn:deltax}
\end{equation}
$$ \bx^{(t)}\leftarrow (1-\eta)\bx^{(t-1)} + \eta \tilde{\bx} $$
\STATE
Update $w$ to maintain the value of $A\bx$:
\begin{equation}
\bw^{(t)}\leftarrow (1-\eta)\bw^{(t-1)}+ \eta A\Delta\bx 
\label{eqn:update_w}
\end{equation}
\STATE
Consider the potential dual update:
\begin{equation} \label{dual_search}
\tilde{\by} =  \argmax_{\by'}\left\{ \frac{1}{n} \langle \bw^{(t)}, \by'\rangle- f^*(\by')-\frac{1}{2\delta}\|\by'-\by^{(t-1)}\|^2 \right\}. 
\end{equation} 
\STATE 
Choose greedily the dual coordinates to update:
let $I^{(t)}$ be the top $k$ coordinates that maximize 
$$ |\tilde{y}_i - y_i^{(t-1)} |, i\in [n].  $$
Update the dual variable accordingly: 
\begin{equation} 
y_i^{(t)}\leftarrow \left\{
\begin{array}{lc}
    \tilde{y}_i &  \text{ if }i\in I^{(t)}\\
    y_i^{(t-1)} & \text{ otherwise.} 
\end{array}
\right.
\end{equation}

\STATE 
Update $z$ to maintain the value of $A^{\top}\by$
\begin{equation}
\bz^{(t)}\leftarrow \bz^{(t-1)}+ A_{:,I^{(t)}}^\top (\by^{(t)}-\by^{(t-1)})
\label{eqn:update_z}
\end{equation}
\ENDFOR
\STATE {\bfseries Output:} $\bx^{(T)}, \by^{(T)}$
\end{algorithmic}
\label{alg:PD_l1}}
\end{algorithm*}
As $\La(\bx,\by)$ is $\mu$-strongly convex and $L$-smooth with respect to $\bx$, we set the primal learning rate $\eta= \frac{\mu}{2L}$ according to Section \ref{sec:vignette}. Meanwhile, the dual learning rate $\delta$ is set to balance its effect on the dual progress as well as the primal progress. We specify it in the theoretical analysis part.


The computational complexity for each iteration in Algorithm \ref{alg:PD_l1} is $\Ocal(ns)$.
Both primal and dual update could be viewed as roughly three steps: coordinate selection, variable update, and maintaining $A^T\by$ or $A\bx$. 
The coordinate selection as Eqn. \eqref{eqn:deltax} for primal and the choice of $I^{(t)}$ for dual variable respectively take $\Ocal(d)$ and $\Ocal(n)$ on average if implemented with the quick selection algorithm. 
Meanwhile, the variable update costs $\Ocal(s)$ and $\Ocal(k)$. 
The dominating cost is maintaining $A\bx$ as in Eqn. \eqref{eqn:update_w} that takes $\Ocal(ns)$, and maintaining $\A^\top \by$ as in Eqn. \eqref{eqn:update_z} that takes $\Ocal(dk)$. To balance the time budget for primal and dual step, we set $k=ns/d$ and achieve an overall complexity of $\Ocal(ns)$ per iteration.


\subsection{Theoretical Analysis}
We derive convergence analysis under Assumption \ref{assump:main}. 
The derivation consists of the analysis on the primal progress, the balance of the dual progress, and their overall effect.

Define the primal gap as $\Delta^{(t)}_p\define \La(\bx^{(t+1)},\by^{(t)})-\La(\bar{\bx}^{(t)},\by^{(t)})$, where $\bar{\bx}^{(t)}$ is the primal optimal solution such that the dual $D(\by^{(t)})=\La(\bar{\bx}^{(t)},\by^{(t)})$, and is sparse enforced by the $\ell_1$ constraint. The dual gap is $\Delta^{(t)}_d\define D^*-D(\by^{(t)})$. 
We analyze the convergence rate of duality gap $\Delta^{(t)}\equiv \max\{1,(\beta/\alpha-1)\} \Delta^{(t)}_p+\Delta^{(t)}_d$. 

{\bf Primal progress: } Firstly, similar to the analysis in Section \ref{sec:vignette}, we could derive that primal update introduces a sufficient descent as in Lemma \ref{lemma:primal_progress}. 
$$\La(\bx^{(t+1)}, \by^{(t)})-\La(\bx^{(t)}, \by^{(t)})\leq -\frac{\eta}{2} \Delta^{(t)}_p.$$

{\bf Dual progress: } With the {\bf GS-r} rule to carefully select and update the most important $k$ coordinates in the dual variable in \eqref{dual_search}, we are able to derive the following result on dual progress that diminishes dual gap as well as inducing error. 
$$-\|\by^{(t)} - \by^{(t-1)}\|^2
    \leq -\frac{k\delta}{n\beta}\Delta_d^{(t)} + \frac{k\delta}{n^2} R\|\bar{\bx}^{(t)}-\bx^{(t)}\|_2^2    $$
Refer to Lemma \ref{lemma:dual_progress} for details.

{\bf Primal Dual progress:} The overall progress evolves as: 
$$\Delta^{(t)}-\Delta^{(t-1)} \leq
 \overbrace{\La(\bx^{(t+1)},\by^{(t)})-\La(\bx^{(t)},\by^{(t)})}^{\text{primal progress}}-\frac{1}{4\delta}\overbrace{\|\by^{(t)}-\by^{(t-1)}\|^2}^{\text{dual progress}}+ \frac{3\delta Rk}{2n^2}\overbrace{\|\bar{\bx}^{(t)}-\bx^{(t)}\|^2}^{\text{primal hindrance}}.$$
In this way, we are able to connect the progress on duality gap with constant fraction of its value, and achieve linear convergence: 
\begin{theorem}
\label{thm:main} 
Given a function $P(\bx)=\sum_{i=1}^n f_i(\ba_i^\top \bx)+g(\bx)$ that satisfies Assumption \ref{assump:main}. Set $s$ to upper bound the sparsity of the primal optimal $\bar{\bx}^{(t)}$, and learning rates $\eta = \frac{\mu}{2L}, \delta = \frac{1}{k}(\frac{L}{\mu n \beta}+\frac{5\beta R}{2\alpha \mu n^2}(1+4\frac{L}{\mu}))^{-1}$.
The duality gap $\Delta^{(t)}=\max\{1,\frac{\beta}{\alpha}-1\}\Delta^{(t)}_p+\Delta^{(t)}_d$ generated by Algorithm \ref{alg:PD_l1} takes $\Ocal(\frac{L}{\mu}(1+\frac{\beta}{\alpha}\frac{R\beta}{n\mu})\log \frac{1}{\epsilon})$ iterations to achieve $\epsilon$ error. The overall complexity is $\Ocal(ns\frac{L}{\mu}(1+\frac{\beta}{\alpha}\frac{R\beta}{n\mu})\log \frac{1}{\epsilon})$.
\end{theorem}
For our target applications like elastic net, or ERM with smooth hinge loss, we are able to connect the time complexity to the condition number of the primal form. 

\begin{corollary}
\label{coro:main}
Given a smooth hinge loss or quadratic loss $f_i$ that is $\beta$-smooth, and $\ell_2$ regularizer $g=\frac{\mu}{2}\|\bx\|^2$. Define the condition number $\kappa=\frac{\beta R}{\mu}$. Setting $s$ upper bounds the sparsity of the primal optimal $\bar{\bx}^{(t)}$, and learning rates $\eta = \frac{1}{2}, \delta = \frac{1}{k}(\frac{1}{n \beta}+\frac{25 R}{2\mu n^2})^{-1}$,
the duality gap $\Delta^{(t)}$ takes $\Ocal((1+\frac{\kappa}{n})\log \frac{1}{\epsilon})$ iterations to achieve $\epsilon$ error. The overall complexity is $\Ocal(s(n+\kappa)\log \frac{1}{\epsilon})$.
\end{corollary}
Our derivation of overall complexity implicitly requires $ns\geq d$ by setting $k=sd/n\geq 1$. This is true for our considered applications like SVM. Otherwise we choose $k=1$ and the complexity becomes $\Ocal(\max\{d,ns\}(1+\frac{\kappa}{n})\log\frac{1}{\epsilon})$.

In Table \ref{table:time_complexity}, we briefly compare the time complexity of our algorithm with some benchmark algorithms: (1) Accelerated Projected Gradient Descent (PGD) (2) Frank-Wolfe algorithm (FW) (3) Stochastic Variance Reduced Gradient (SVRG) \cite{rie2013acceler} (4) Stochastic Conditional Gradient Sliding (SCGS) \cite{lan2016conditional} and (5) Stochastic Variance-Reduced Conditional Gradient Sliding (STORC) \cite{hazan2016variance}. The comparison is not thorough but intends to select constrained optimization that improves the overall complexity from different perspective. Among them, accelerated PGD improves conditioning of the problem, while SCGS and STORC reduces the dependence on number of samples. In the experimental session we show that our proposal outperforms the listed algorithms under various conditions.
\begin{table}[h]
\begin{center}
\begin{tabular}{|c|c|c|}
\hline
Algorithm        & Per Iteration Cost & Iteration Complexity\\ \hline
Frank Wolfe
 & $\Ocal(nd)$ & $\Ocal(\frac{1}{\epsilon})$ \\ \hline
Accelerated PGD  \cite{nesterov2013introductory}
 & $\Ocal(nd)$ & $\Ocal(\sqrt{\kappa}\log\frac{1}{\epsilon})$ \\ \hline
 SVRG \cite{rie2013acceler}
 & $\Ocal(nd)$ & $\Ocal((1+\kappa/n)\log\frac{1}{\epsilon})$ \\ \hline
SCGS \cite{lan2016conditional}
          & $\Ocal(\kappa^2 \frac{\#\text{iter}^3}{\epsilon^2}d)$ & $\Ocal(\frac{1}{\epsilon})$ \\ \hline
STORC \cite{hazan2016variance}
         & $\Ocal(\kappa^2d+nd)$ & $\Ocal(\log\frac{1}{\epsilon}) $ \\ \hline
Primal Dual FW (ours)   
& $\Ocal(ns)$ & $\Ocal((1+\kappa/n)\log\frac{1}{\epsilon})$ \\ \hline
\end{tabular}
\end{center}
\caption{Time complexity comparisons on the setting of Corollary \ref{coro:main}. For clear comparison, we refer the per iteration cost as the time complexity of outer iterations. 
}
\label{table:time_complexity}
\vspace{-15pt}
\end{table}
\subsection{Extension to the Trace Norm Ball}
\label{sec:trace_norm}

\begin{algorithm*}[thb]
{\footnotesize
\caption{Frank-Wolfe Block Primal-Dual Method for Trace Norm Ball}
\begin{algorithmic}[1]
\STATE {\bfseries Input:}  Training data $A\in \mathbb{R}^{n\times d}$, primal and dual step size $\eta,\delta>0$. Target accuracy $\epsilon$.
\STATE {\bfseries Initialize:}  $X^{(0)}\leftarrow0\in \mathbb{R}^{d\times c}$, $Y^{(0)}\leftarrow0\in \mathbb{R}^{n\times c}, W^{(0)}\equiv AX=0\in \R^{n\times c}, Z^{(0)}\equiv A^{\top}Y=0\in\R^{d\times c} $
\FOR{$t=1,2,\cdots, T$}
\STATE 
Use Frank Wolfe to Update the primal variable: 
$$X^{(t)}\leftarrow (1-\eta)X^{(t-1)} + \eta \tilde{X}, \text{ where }\tilde{X} \leftarrow (\frac{1}{2},\frac{\epsilon}{8}) \text{-approximation of Eqn. \eqref{eqn:deltaX}}. $$
\STATE
Update $W$ to maintain the value of $AX$:
\begin{equation}
W^{(t)}\leftarrow (1-\eta)W^{(t-1)}+ \eta A\tilde{X}
\label{eqn:update_W}
\end{equation}
\STATE 
Consider the potential dual update:
\begin{equation} \label{dual_update_trace}
\tilde{Y}^{(t)}\leftarrow \argmax_{Y} \left\{\langle W,Y\rangle-f^*(Y) -\frac{1}{2\delta}\|Y-Y^{(t-1)}\|^2 \right\}
\end{equation}
\STATE 
Choose greedily the rows of the dual variable to update: let $I^{(t)}$ be the top $k$ coordinates that maximize
$$ \left\|\tilde{Y}_{i,:}-Y^{(t-1)}_{i,:}\right\|_2,i\in [n]. $$
Update the dual variable accordingly:
\begin{equation} 
Y_{i,:}^{(t)}\leftarrow \left\{
\begin{array}{lc}
    \tilde{Y}_{i,:} &  \text{ if }i\in I^{(t)}\\
    Y_{i,:}^{(t-1)} & \text{ otherwise.} 
\end{array}
\right.
\end{equation}
\STATE 
Update $Z$ to maintain the value of $A^{\top}Y$
\begin{equation}
Z^{(t)}\leftarrow Z^{(t-1)}+ A^\top (Y^{(t)}-Y^{(t-1)})
\label{eqn:update_Z}
\end{equation}
\ENDFOR
\STATE {\bfseries Output:} $X^{(T)}, Y^{(T)}$
\end{algorithmic}
\label{alg:PD_tracenorm}}
\end{algorithm*}

We also extend our algorithm to matrix optimization over trace norm constraints: 
\begin{equation}
\min_{\|X\|_*\leq \lambda, X\in\R^{d\times c}}\left\{ \frac{1}{n}\sum_{i=1}^n f_i(\ba_i^\top X)+g(X)\right\}.
\label{eqn:primal_trace}
\end{equation}
This formulation covers multi-label multi-class problems, matrix completion, affine rank minimization, and phase retrieval problems (see reference therein \cite{candes2015phase,allen2017linear}). Equivalently, we solve the following primal-dual problem:
$$\min_{\|X\|_*\leq \lambda, X\in \R^{d\times c}}\max_{Y\in \R^{n\times c}} \left\{\La(X,Y)\equiv g(X)+\frac{1}{n}\langle AX,Y\rangle - \frac{1}{n}\sum_{i=1}^n f^*_i(\by_i)\right\}.$$ Here $\by_i$ is the $i$-th row of the dual matrix $Y$. 
For this problem, the \textit{partial update} we enforced on the primal matrix is to keep the update matrix low rank: 
\begin{equation}
 \tilde{X}  \leftarrow
 \hspace{-5pt}
 \argmin_{\|X\|_*\leq \lambda, \text{rank}(X)\leq s}  \left\{\langle \frac{1}{n}Z +\nabla g(X^{(t-1)}), X\rangle +\frac{L}{2} \eta \|X - X^{(t-1)}\|^2\right\}, Z\equiv A^\top Y^{(t-1)}.
 \label{eqn:deltaX}
\end{equation}
However, an exact solution to \eqref{eqn:deltaX} requires computing the top $s$ left and right singular vectors of the matrix $X^{(t-1)}-\frac{1}{\eta L}(Z+\nabla g(X^{(t-1)})\in \R^{d\times c}$. 
Therefore we loosely compute an $(\frac{1}{2},\epsilon/2)$-approximation, where $\epsilon$ is the target accuracy, based on the following definition:
\begin{definition}[Restated Definition 3.2 in \cite{allen2017linear}]
Let $l_t(V) =\langle \nabla_X \La(X^{(t)},Y^{(t)}), V - X^{(t)}\rangle + \frac{L}{2}\eta \|V - X^{(t)}\|_F^2$ be the objective function in \eqref{eqn:deltaX}, and let $l^*_t = l_t(\bar{X}^{(t)})$. Given parameters $\gamma \geq 0$ and $\epsilon \geq 0$, a feasible solution $V$ to \eqref{eqn:deltaX} is called $(\gamma, \epsilon)$-approximate if it satisfies $l(V) \leq (1 - \gamma)l^*_t + \epsilon$.
\end{definition}
The time dependence on the data size $n,c,d,s$ is $ncs+s^2(n+c)$ \cite{allen2017linear}, and is again independent of $d$. Meanwhile, the procedures to keep track of $ W^{(t)}\equiv AX^{(t)}$ requires complexity of $nds+ncs$, while updating $Y^{(t)}$ requires $ dck$ operations. Therefore, by setting $k\leq ns(1/c+1/d)$, the iteration complexity's dependence on the data size becomes $\Ocal(n(d+c)s)$ operations, instead of $\Ocal(ndc)$ for conducting a full projected gradient step. Recall that $s$ upper bounds the rank of $\bar{X}^{(t)}\leq \min\{d,c\}$.

The trace norm version mostly inherits the convergence guarantees for vector optimization. Refer to the Appendix for details. 

\begin{assumption}
We assume the following property on the primal form \eqref{eqn:primal_trace}:
\begin{itemize}[topsep=0pt,parsep=0pt,partopsep=0pt]
    \item  $f_i$ is convex, and $\beta$-smooth. Its convex conjugate $f_i^*$ exists and satisfies $\frac{1}{\alpha}$-smooth on some convex set (could be $\R^c$) and infinity otherwise.
\item Data matrix $A$ satisfies $R=\max_{|I|\leq k, I\subset [n]}\sigma_{\max}^2(A_{I,:})$ ($\leq \|A\|_2^2$). Here $\sigma_{\max}(X)$ denotes the largest singular value of $X$.
\item  $g$ is $\mu$-strongly convex and $L$-smooth. 
\end{itemize}
\label{assump:trace}
\end{assumption}
The assumptions also cover smooth hinge loss as well as quadratic loss. With the similar assumptions, the convergence analysis for Algorithm \ref{alg:PD_tracenorm} is almost the same as Algorithm \ref{alg:PD_l1}. The only difference comes from the primal step where approximated update produces some error:

{\bf Primal progress: } With the primal update rule in Algorithm \ref{alg:PD_tracenorm}, it satisfies 
$\La(X^{(t+1)},Y^{(t)}) -  \La(X^{(t)},Y^{(t)})\leq -\frac{\mu}{8L} \Delta_p^{(t)}+\frac{\epsilon}{16}$. (See Lemma \ref{lemma:primal_progress_trace}.)
With no much modification in the proof, we are able to derive similar convergence guarantees for the trace norm ball. 
\begin{theorem}
\label{thm:main_trace}
Given a function $\frac{1}{n}\sum_{i=1}^n f_i(\ba_i^\top X)+g(X)$ that satisfies Assumption \ref{assump:trace}. Setting $s\geq$rank$(\bar{X}^{(t)})$, and learning rate $\eta=\frac{\mu}{2L}, \delta \leq \frac{1}{k}(\frac{L}{\mu n \beta}+\frac{5\beta R}{2\alpha \mu n^2}(1+8\frac{L}{\mu}))^{-1}$, the duality gap $\Delta^{(t)}$ generated by Algorithm \ref{alg:PD_tracenorm} satisfies $\Delta^{(t)}\leq \frac{k\delta}{k\delta+8\beta n}\Delta^{(t-1)}+\frac{\epsilon}{16}. $ Therefore it takes $\Ocal(\frac{L}{\alpha}(1+\frac{\beta}{\alpha}\frac{R\beta}{n\mu})\log\frac{1}{\epsilon} )$ iterations to achieve $\epsilon$ error. 
\end{theorem}

We also provide a brief analysis on the difficulty to extend our algorithm to polytope-type constraints in the Appendix \ref{sec:polytope}. 

\section{Experiments}

\begin{figure*}[thb]
\begin{center}
\includegraphics[width=0.34\linewidth]{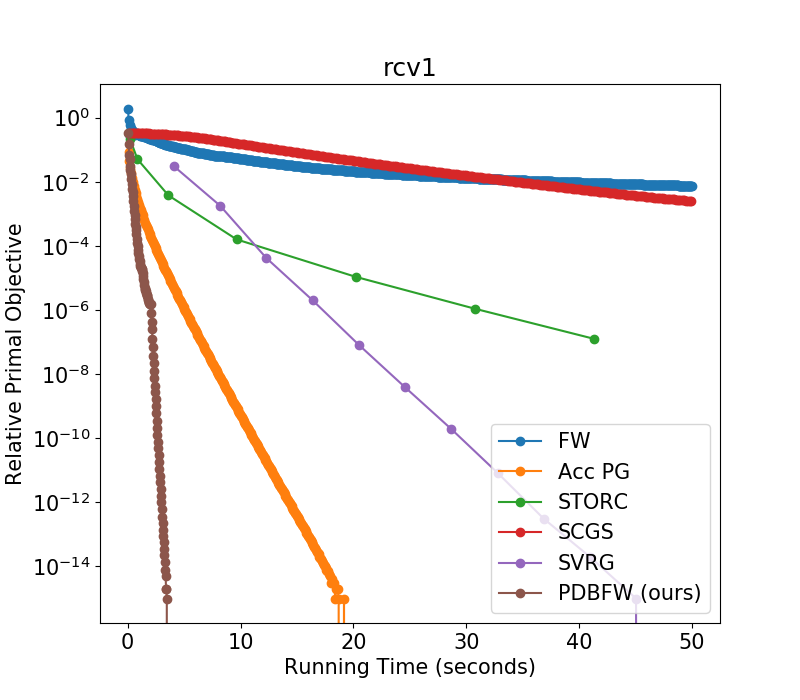}
\hspace{-15pt}
\includegraphics[width=0.34\linewidth]{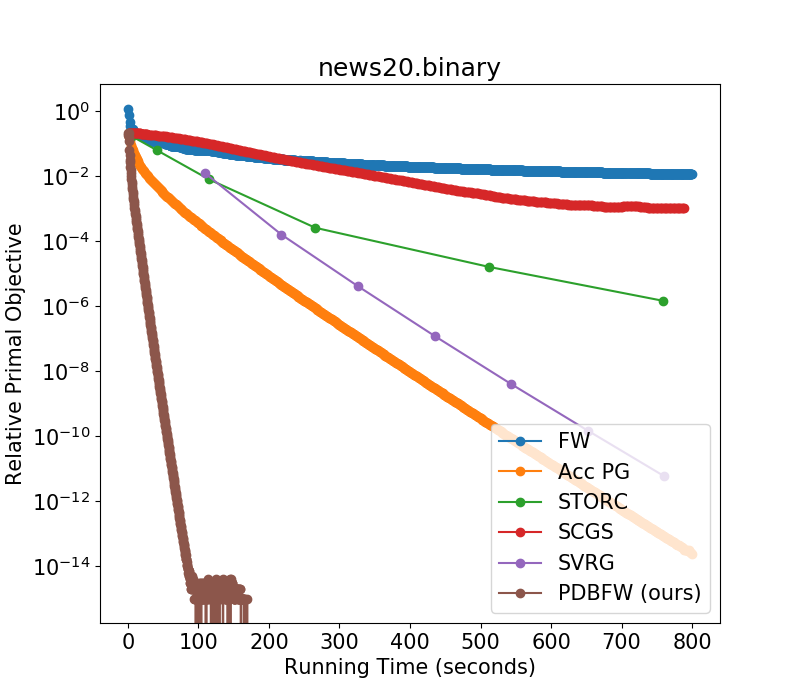}
\hspace{-15pt}
\includegraphics[width=0.34\linewidth]{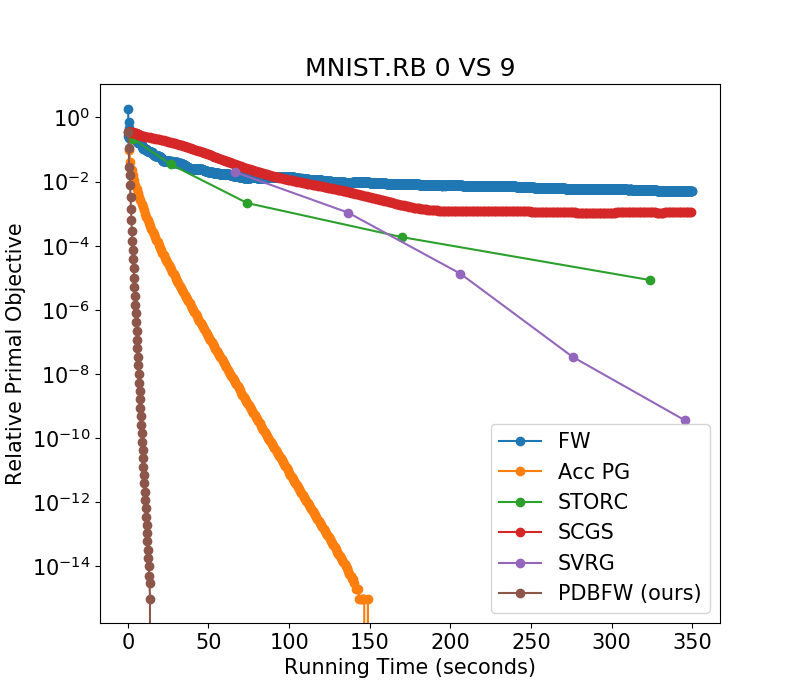}
\vspace{-15pt}

\includegraphics[width=0.34\linewidth]{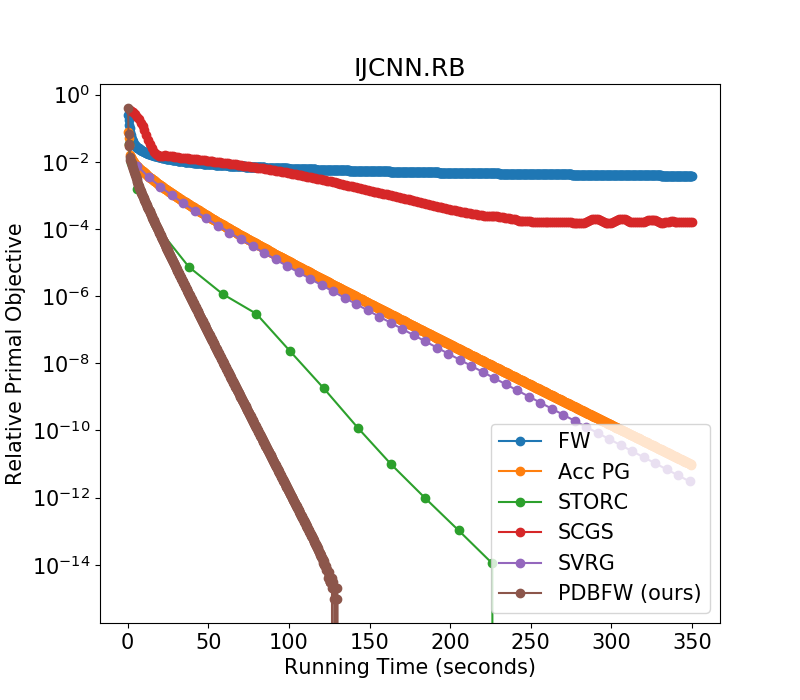}
\hspace{-15pt}
\includegraphics[width=0.34\linewidth]{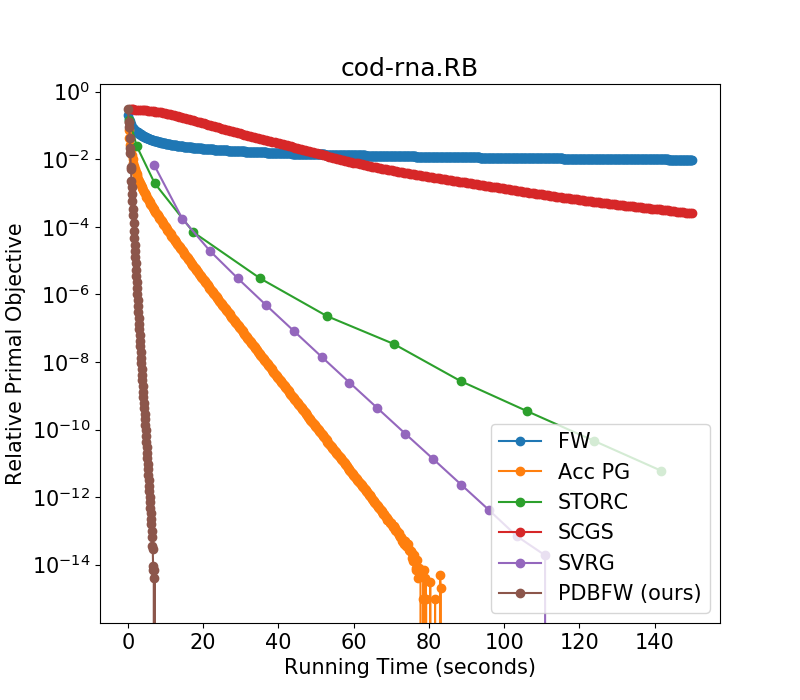}
\hspace{-15pt}
\includegraphics[width=0.34\linewidth]{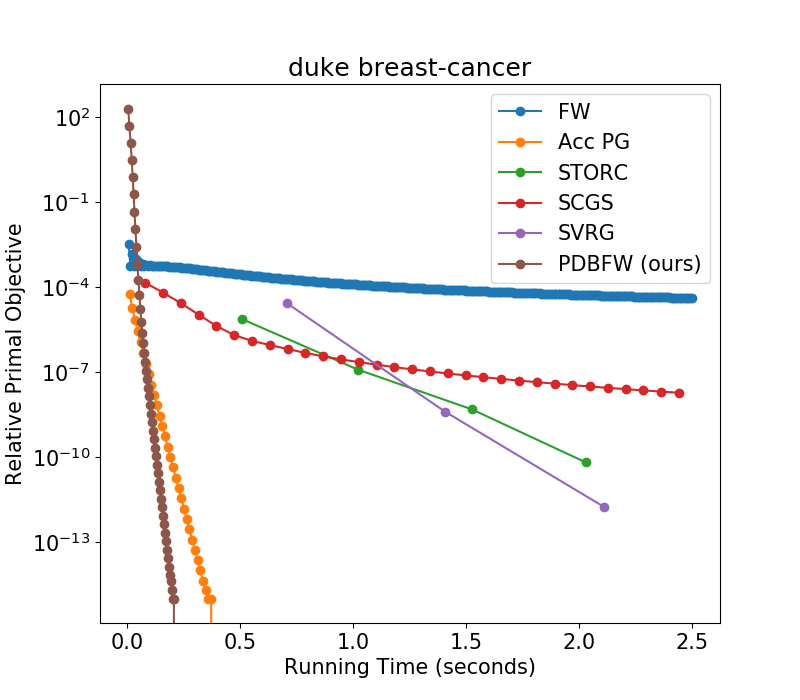}
\end{center}
\vspace{-10pt}
\caption{{\bf Convergence result comparison of different algorithms on smoothed hinge loss.} For six different datasets, we show the decrease of relative primal objective: $({P(\bx^{(t)})-P^*})/{P^*}$ over CPU time. Our algorithm (brown) achieves around 10 times speedup over all other methods except for the smallest dataset duke.}
\label{fig:hinge_loss} 
\vspace{-5pt}
\end{figure*}

We evaluate the primal-dual Frank-Wolfe algorithm by its performance on binary classification with smoothed hinge loss. We refer the readers to Appendix \ref{sec:smooth_hinge_loss} for details about smoothed hinge loss.

We compare the proposed algorithm against five benchmark algorithms:  (1) Accelerated Projected Gradient Descent (Acc PG) (2) Frank-Wolfe algorithm (FW) (3) Stochastic Variance Reduced Gradient (SVRG) \cite{rie2013acceler} (4) Stochastic Conditional Gradient Sliding (SCGS) \cite{lan2016conditional} and (5) Stochastic Variance-Reduced Conditional Gradient Sliding (STORC) \cite{hazan2016variance}. We presented the time complexity for each algorithm in Table \ref{table:time_complexity}. Three of them (FW, SCGS, STORC) are projection-free algorithms, and the other two (Acc PG, SVRG) are projection-based algorithms. Algorithms are implemented in C++, with the Eigen linear algebra library \cite{eigenweb}. 


The six datasets used here are summarized in Table \ref{table:data}. All of them can be found in LIBSVM datasets \cite{chang2011libsvm}. 
We augment the features of MNIST, ijcnn, and cob-rna by random binning \cite{rahimi2008random}, which is a standard technique for kernel approximation. 
Data is normalized. 
We set the $\ell_1$ constraint to be $300$ and the $\ell_2$ regularize parameter to $10/n$ to achieve reasonable prediction accuracy. 
We refer the readers to the Appendix \ref{sec:more_l1_result} for results of other choice of parameters. These datasets have various scale of features, samples, and solution sparsity ratio.

The results are shown in Fig \ref{fig:hinge_loss}. To focus on the convergence property, we show the decrease of loss function instead of prediction accuracy. From Fig \ref{fig:hinge_loss}, our proposed algorithm consistently outperforms the benchmark algorithms. The winning margin is roughly proportional to the solution sparsity ratio, which is consistent with our theory.


\begin{table}[]
\begin{center}
\begin{tabular}{|c|c|c|c|c|c|}
\hline
Dataset Name       & \# Features & \# Samples  & \# Non-Zero & Solution Sparsity (Ratio) \\ \hline
duke breast-cancer \cite{chang2011libsvm} 
& 7,129        & 44                   & 313,676           & 423 (5.9\%)         \\ \hline
rcv1 \cite{chang2011libsvm}               
& 47,236       & 20,242               & 1,498,952         & 1,169 (2.5\%)          \\ \hline
news20.binary  \cite{chang2011libsvm}     
& 1,355,191    & 19,996               & 9,097,916         & 1,365 (0.1\%)         \\ \hline
MNIST.RB 0 VS 9  \cite{chang2011libsvm, rahimi2008random}  
& 894,499      & 11,872               & 1,187,200         & 8,450 (0.9\%)         \\ \hline
ijcnn.RB  \cite{chang2011libsvm, rahimi2008random}  
& 58,699      & 49,990               & 14,997,000         & 715 (1.2\%)     \\ \hline
cob-rna.RB  \cite{chang2011libsvm, rahimi2008random}  
& 81,398      & 59,535               & 5,953,500         & 958 (1.2\%)     \\ \hline
\end{tabular}
\end{center}
\caption{Summary of the properties of the datasets.}
\label{table:data}
\end{table}

\subsection{Experiments with trace norm ball on synthetic data}
\label{sec:trace_norm_result}
For trace norm constraints, we also implemented our proposal Primal Dual Block Frank Wolfe to compare with some prior work, especially Block FW \cite{allen2017linear}. Since prior work were mostly implemented in Matlab to tackle trace norm projections, we therefore also use Matlab to show fair comparisons. We choose quadratic loss $f(AX)=\|AX-B\|_F^2$ and $g$ to be $\ell_2$ regularizer with $\mu=10/n$. The synthetic sensing matrix $A\in\R^{n\times d}$ is dense with $n=1000$ and $d=800$. Our observation $B$ is of dimension $1000\times 600$ and is generated by a ground truth matrix $X_0$ such that $B=AX_0$. Here $X_0\in \R^{800\times 600}$ is constructed with low rank structure. We vary its rank $s$ to be $10, 20,$ and $100$. The comparisons with stochastic FW, blockFW \cite{allen2017linear}, STORC \cite{hazan2016variance}, SCGS \cite{lan2016conditional}, and projected SVRG \cite{rie2013acceler} are presented in Figure \ref{fig:tracenorm}, which verifies that our proposal PDBFW consistently outperforms the baseline algorithms.  

\begin{figure*}
\begin{center}
\includegraphics[width=0.34\linewidth]{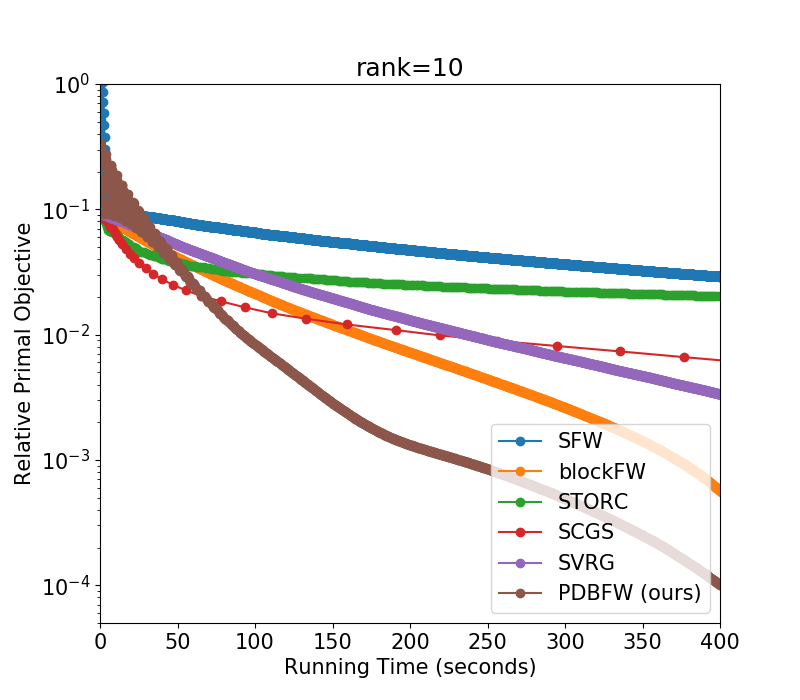}
\hspace{-15pt}
\includegraphics[width=0.34\linewidth]{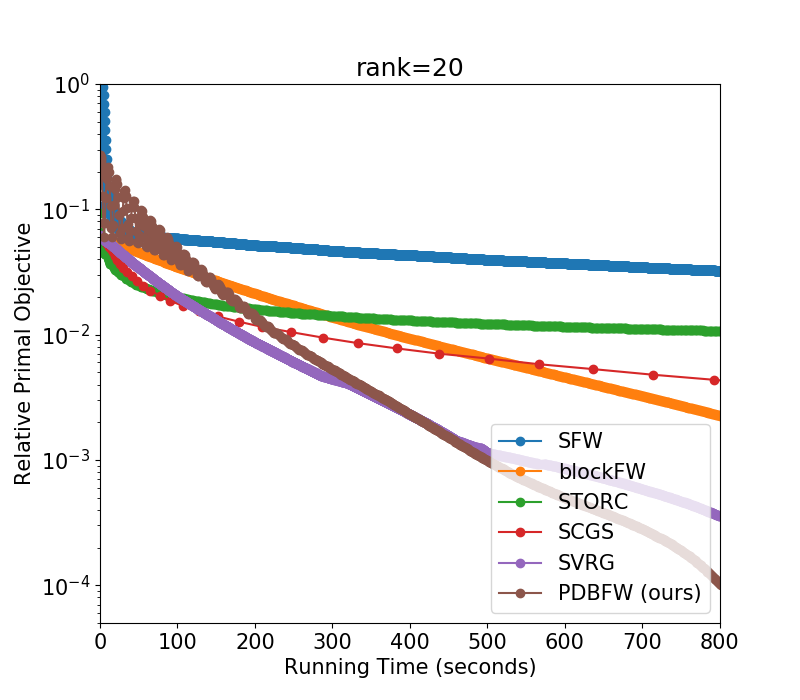}
\hspace{-15pt}
\includegraphics[width=0.34\linewidth]{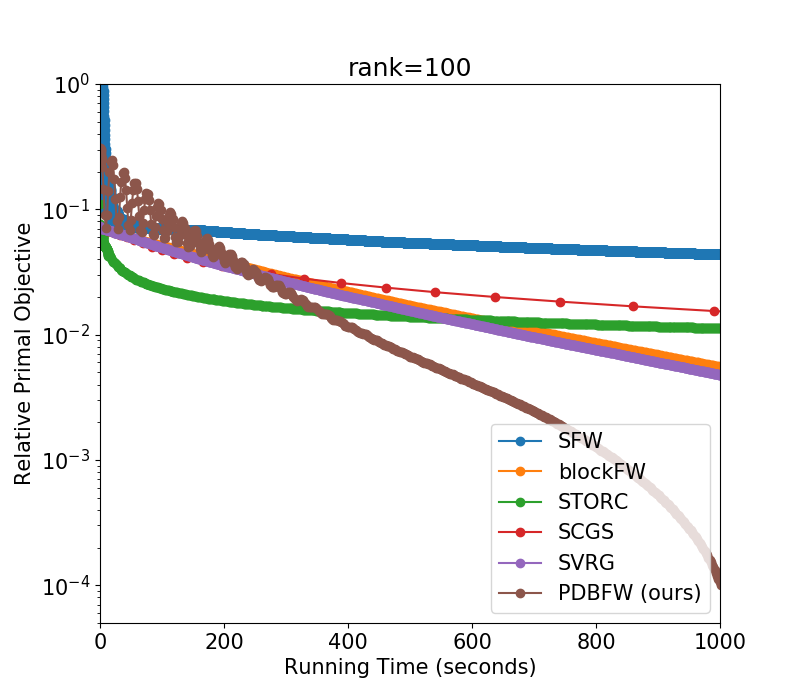}
\end{center}
\caption{Convergence comparison of our Primal Dual Block Frank Wolfe and other baselines. Figures show the relative primal objective value decreases with the wall time.}
\label{fig:tracenorm} 
\end{figure*}

\section{Conclusion}
In this paper we consider a class of problems whose solutions enjoy some simple structure induced by the constraints. We propose a FW type algorithm to exploit the simple structure and conduct partial updates, reducing the time cost for each update remarkably while attaining linear convergence.  For a class of ERM problems, our running time depends on the sparsity/rank of the optimal solutions rather than the ambient feature dimension. 
Our empirical studies verify the improved performance compared to various state-of-the-art algorithms.

\bibliography{main}
\bibliographystyle{plain}

\appendix
\include{appendix} 
\end{document}

%% file: appendix.tex
\section{Efficient Coordinate Selections}
\label{sec:coordinate_selection}
The modified Block Frank-Wolfe step in Eqn. \eqref{eqn:block_frank_wolfe} achieves an $s$-sparse update of the iterates and could be computed efficiently when one knows which $s$ coordinates to update. However, in order to find the $s$ coordinates, one needs to compute the full gradient $\nabla f(\bx)$ with naive implementation. This phenomenon reminds us of greedy coordinate descent.

Even with the known fact that coordinate descent converges faster with greedy selection than with random order\cite{nutini2015coordinate}, there have been hardness to propogate this idea because of expensive greedy selections since the arguments that GCD converges similarly with RCD in \citep{nesterov2012efficiency}, except for special cases \cite{lei2016coordinate,lei2017doubly,dhillon2011nearest,karimireddy2018efficient}. This is also probability why the partial updates nature of FW steps is less exploited before.

We investigate some possible tricks to boost GCD method that could be possibly applied to FW methods. A recent paper \citep{karimireddy2018efficient}, Karimireddy et al. make connections between the efficient choice of the greedy coordinates with the problem of Maximum Inner Product Search (MIPS) for a composite function $P(\bx)=f(A\bx)+g(\bx)$, where $A\in \R^{n\times d}$. We rephrase the connection for the Frank-Wolfe algorithm. Since the computation of gradient is essentially $A^\top \nabla f_{|A\bx} + \nabla g(\bx)$, to find its largest magnitude is to search maximum inner products among:
$$ \pm \langle [\tilde{\ba}_i^\top|1] , [\nabla f_{|A\bx}^\top |\nabla_i g(\bx)]\rangle, \text{i.e.} \pm\left(\tilde{\ba}_i^\top \nabla f_{|A\bx} + \nabla_i g(\bx)\right), $$ 
where $\tilde{\ba}_i\in \R^n$ is the $i$-th column of data matrix $A$, and $\nabla f_{|A\bx}$ is the gradient of $f$ at $A\bx$. In this way, we are able to select the greedy coordinates by conducting MIPS for a fixed $\R^{2d\times (n+1)}$ matrix $[A^\top|I| -A^\top|-I]^\top$ and each newly generated vector $[\nabla f_{|A\bx}^\top|\nabla g_i(\bx)]$. Therefore when $\nabla g_i$ is constant for linear function or $\pm \lambda$ for $g(\bx)=\lambda\|\bx\|_1$, we could find the largest magnitude of the gradient in sublinear time. Still, the problems it could conquer is very limited. It doesn't even work for $\ell_2$ regularizer since the different coordinates in $\nabla_i g(\bx)$ creates $d$ new vectors in each iteration and traditional MIPS could resolve it in time sublinear to $d$. Meanwhile, even with constant $\nabla_i g(\bx)$, it still requires at least $\Ocal((2d)^c\log(d))$ times of inner products of dimension $n+1$ for some constant $c$ \cite{shrivastava2014asymmetric}. 

However, we have shown that for general composite form $ f(A \bx)+g(\bx)$ with much more relaxed requirements on the regularizer $g$, we are able to select and update each coordinate with {\em constant} times of inner products on average while achieving linear convergence. Therefore the usage of these tricks applied on FW method (MIPS as well as the nearest neighbor search \cite{dhillon2011nearest}) is completely dominated by our contribution and we omit them in the main text of this paper. 
\newpage
\section{Proofs}
\subsection{Derivation of Primal-Dual Formulation} 
\begin{eqnarray*}
 &&    \min_{\bx\in C} P(\bx) = \frac{1}{n}\sum_{i=1}^nf_i(\ba_i^\top\bx) + g(\bx) \\
 &=&  \min_{\bx\in C, \bb=A\bx} \frac{1}{n}\sum_{i=1}^nf_i(b_i) + g(\bx)\\
    & = & \min_{\bx\in C, \bb }  \max_{\by }\left\{  \frac{1}{n}\sum_{i=1}^nf_i(b_i) + g(\bx) + \frac{1}{n}\langle \by,A\bx-\bb \rangle \right\} \\
&=  & \min_{\bx\in C} \max_{\by}\left\{ g(\bx) + \frac{1}{n}\langle \by, A\bx \rangle  + \min_{\bb}\left\{ \frac{1}{n}\sum_{i=1}^nf_i(b_i) - \frac{1}{n}\langle \by,  \bb \rangle \right\}\right\} \\
&=& \min_{\bx\in C } \max_{\by }\left\{\La(\bx,\by):= g(\bx) + \frac{1}{n}\langle \by, A\bx\rangle - \frac{1}{n}\sum_{i=1}^n f^*_i(y_i) \right\} \\
&=&  \max_{\by }\left\{ D(\by):= \min_{\bx\in C } \left\{g(\bx) + \frac{1}{n}\langle \by, A\bx\rangle - \frac{1}{n}\sum_{i=1}^n f^*_i(y_i)\right\} \right\} 
\end{eqnarray*}
We use Von Neumann-Fan minimax theorem for the whole derivation when swapping each min-max formula \cite{du2013minimax}. 
For the last equality, there is a convex constraint in the minimization part. Although the original Von Neumann-Fan doesn't have constraints, it naturally applies to the case when $\bx$ (assuming function is convex to $\bx$) is bounded in a convex set, since we could change $f(\bx,\by)$ to $f(\bx,\by) + I_C(\bx)$, where $I_C(\bx)=0$ if $\bx\in C$ and $\infty$ otherwise. Then the property will be properly inherited.

\subsection{Notation and simple facts}
Recall primal, dual and Lagrangian forms:
\begin{eqnarray*}
P(\bx) &\define& \frac{1}{n}\sum_{i=1}^n f_i(\ba_i^\top\bx) + g(\bx) \\
\La(\bx,\by) &\define& g(\bx) + \frac{1}{n}\by^\top A \bx -\frac{1}{n}\sum_{i=1}^nf^*_i(y_i) \\
D(\by) &\define& \min_{\bx\in C} \La(\bx,\by) = \La(\bar{\bx}(\by),\by)
\end{eqnarray*}

Similar to the definitions in \cite{lei2017doubly}, we introduce the primal gap defined as $\Delta^{(t)}_p\define \La(\bx^{(t+1)},\by^{(t)})-D(\by^{(t)})$, and dual gap $\Delta^{(t)}_d\define D^*-D(\by^{(t)})$. Recall the assumptions:
\begin{itemize}
    \item $f_i$ is convex and $\beta$-smooth, and is $\alpha$ strongly convex over some convex set, and linear otherwise.
    \item $R=\max_i\|\ba_i\|_2^2, \forall i\in [n]$.
    \item $g$ is $\mu$-strongly convex and $L$-smooth. 
\end{itemize}
To begin with, it is easy to verify that $f_i^*$ is $1/\beta$-strongly convex and is $1/\alpha$-smooth on a convex set and infinity otherwise (See Claim \ref{claim:indicator}).
For simplicity we first assume $\alpha\geq \frac{1}{2}\beta$ and then generalize the result. 
\begin{claim}
\begin{itemize} 
\item Since $D(\by) = \min_{\bx\in C} \{g(\bx) + \frac{1}{n}\by^\top A \bx\} -\frac{1}{n}\sum_{i=1}^nf^*_i(\by)$, $-D(\by)$ is at least $\frac{1}{\beta}$-strongly convex. 
\item Based on our update rule, $\exists \bg \in \partial_{\by}\frac{1}{n}\sum_{i} f_i^*(\by^{(t)})$, such that 
\begin{equation}
    \label{eqn:update_rule_g}
\by^{(t)}_{I^{(t)}}-\by^{(t-1)}_{I^{(t)}} = \delta (\frac{1}{n}A_{I^{(t)},:}\bx^{(t)}-\bg_{I^{(t)}}).
\end{equation}
\end{itemize} 
And our update rule ensures that $I^{(t)}$ consists of indices $i\in [n]$ that maximizes $|\frac{1}{n}\ba_i^\top \bx^{(t)}-g_i|$.
\end{claim}

\subsection{Primal Progress} 

\begin{lemma}(Primal Progress)
\label{lemma:primal_progress}
\begin{eqnarray*}
\La(\bx^{(t+1)}, \by^{(t)}) - \La(\bar{\bx}^{(t)}, \by^{(t)} ) \leq (1-\frac{\eta}{2}) \left( \La(\bx^{(t)}, \by^{(t)}) - \La(\bar{\bx}^{(t)}, \by^{(t)} ) \right)
\end{eqnarray*}
Or equivalently, 
$$(1-\frac{\eta}{2})(\La(\bx^{(t+1)}, \by^{(t)})-\La(\bx^{(t)}, \by^{(t)})) \leq -\frac{\eta}{2} \left( \La(\bx^{(t+1)}, \by^{(t)}) - \La(\bar{\bx}^{(t)}, \by^{(t)} ) \right) \equiv -\frac{\eta}{2} \Delta^{(t)}_p $$
\begin{proof}
Simply replace $h_t$ as $\La(\bx^{(t)},\by^{(t)})-D(\by^{(t)})$ and $h_{t+1}$ as $\La(\bx^{(t+1)},\by^{(t)})-D(\by^{(t)})$ in Inequality \eqref{eqn:fw_progress}. We could conclude that $h_{t+1}\leq (1-\eta +\eta^2 \frac{L}{\mu})h_t$. Therefore when $\eta\leq \frac{\mu}{2L}$, $h_{t+1}\leq (1-\frac{\eta}{2})h_t$ and the first part of Lemma \ref{lemma:primal_progress} is true. Some simple rearrangement suffices the second part of the lemma. 
\end{proof}

\end{lemma}
\subsection{Primal Dual Progress}
In order to get a clue on how to analyze the dual progress, we first look at how the primal and dual evolve through iterations. \\
For an index set $I$ and a vector $\by\in \R^{n}$, denote $\by_{I}=\sum_{i\in I} y_i\be_i\in \R^k$ as the subarray of $\by$ indexed by $I$, with $|I|=k$. Recall Algorithm \ref{alg:PD_l1} selects the coordinates to update in the dual variable as $I^{(t)}$.
\begin{lemma}(Primal-Dual Progress).
	\label{lemma:primal_dual}
\begin{eqnarray*}
    \nonumber
    &&\Delta_d^{(t)}-\Delta_d^{(t-1)}+\Delta_p^{(t)}-\Delta_p^{(t-1)}\\
 &\leq&
 \La(\bx^{(t+1)},\by^{(t)})-\La(\bx^{(t)},\by^{(t)})-\frac{1}{2\delta}\|\by^{(t)}-\by^{(t-1)}\|^2\\
    &&+\frac{2\delta Rk}{n^2}\|\bar{\bx}^{(t)}-\bx^{(t)}\|^2.
    \end{eqnarray*}
\end{lemma}
\begin{proof}
Notice we have claimed that $-D(\by)$ is $\frac{1}{\beta}$-strongly convex and for all $\bg\in \partial_{\by} \frac{1}{n}\sum_{i}^n f^*_i(\by^{(t)})$, 
\begin{eqnarray}
\nonumber
&&\Delta_d^{(t)}-\Delta_d^{(t-1)} = \big(-D(\by^{(t)})\big)-\big(-D(\by^{(t-1)})\big)\\
\nonumber 
&\leq& \langle -\nabla_{\by} \La(\bar{\bx}^{(t)},\by^{(t)}), \by^{(t)}-\by^{(t-1)}\rangle-\frac{1}{2\beta}\|\by^{(t)}-\by^{(t-1)}\|^2 \\
\label{eqn:dual_gap}
&= &  -\langle \frac{1}{n} A_{I^{(t)},:}\bar{\bx}^{(t)}-\bg_{I^{(t)}}, \by^{(t)}_{I^{(t)}}-\by^{(t-1)}_{I^{(t)}}\rangle-\frac{1}{2\beta}\|\by^{(t)}-\by^{(t-1)}\|^2
\end{eqnarray}
Meanwhile since $-\La(\bx,\by)$ is $\frac{1}{\alpha}$-smooth over its feasible set, 
\begin{eqnarray}
\nonumber
&&\La(\bx^{(t)},\by^{(t)})-\La(\bx^{(t)},\by^{(t-1)})\\
\nonumber 
&=&-\La(\bx^{(t)},\by^{(t-1)})-(-\La(\bx^{(t)},\by^{(t)}))\\
\nonumber
&\leq& (\frac{1}{n}A_{I^{(t)},:}\bx^{(t)}-\bg_{I^{(t)}})^\top(\by^{(t)}_{I^{(t)}}-\by^{(t-1)}_{I^{(t)}})+\frac{1}{2\alpha}\|\by^{(t-1)}_{I^{(t)}}-\by^{(t)}_{I^{(t)}}\|^2\\
\label{eqn:update_y}
&=& (\frac{1}{\delta}+\frac{1}{2\alpha})\|\by^{(t)}-\by^{(t-1)}\|^2.
\end{eqnarray}

Also, with the update rule of dual variables, we could make use of Eqn. \eqref{eqn:update_rule_g} and re-write Eqn. \eqref{eqn:dual_gap} as:
\begin{eqnarray}
\nonumber
&&\Delta_d^{(t)}-\Delta_d^{(t-1)} \\
\nonumber
&\leq &  -\langle \frac{1}{n}A_{I^{(t)},:}\bar{\bx}^{(t)}-\bg_{I^{(t)}}, \by^{(t)}_{I^{(t)}}-\by^{(t-1)}_{I^{(t)}}\rangle - \frac{1}{\delta}\|\by^{(t)}-\by^{(t-1)}\|^2\\
\nonumber
&& + (\by^{(t)}-\by^{(t-1)})^\top (\frac{1}{n}A_{I^{(t)},:}\bx^{(t)}-\bg_{I^{(t)}}) - \frac{1}{2\beta}\|\by^{(t)}-\by^{(t-1)}\|^2\\
&=&-\langle \frac{1}{n}A_{I^{(t)},:}(\bar{\bx}^{(t)}-\bx^{(t)}), \by^{(t)}_{I^{(t)}}-\by^{(t-1)}_{I^{(t)}}\rangle- (\frac{1}{\delta}+\frac{1}{2\beta} )\|\by^{(t)}-\by^{(t-1)}\|^2
\label{eqn:delta_d}
\end{eqnarray}
Together we get:
\begin{align*}
    \nonumber
    & \Delta_d^{(t)}-\Delta_d^{(t-1)}+\Delta_p^{(t)}-\Delta_p^{(t-1)}\\
    = & \La(\bx^{(t+1)},\by^{(t)})-\La(\bx^{(t)},\by^{(t)}) + \La(\bx^{(t)},\by^{(t)})-\La(\bx^{(t)},\by^{(t-1)}) \\
    & +2(\Delta_d^{(t)}-\Delta_d^{(t-1)})\\
    \leq& \La(\bx^{(t+1)},\by^{(t)})-\La(\bx^{(t)},\by^{(t)}) 
    +(\frac{1}{\delta}+\frac{1}{2\alpha})\|\by^{(t-1)}_{I^{(t)}}-\by^{(t)}_{I^{(t)}}\|^2+2(\Delta_d^{(t)}-\Delta_d^{(t-1)}) \longrnote{(from Eqn. \eqref{eqn:update_y})} \\
\leq & \La(\bx^{(t+1)},\by^{(t)})-\La(\bx^{(t)},\by^{(t)}) +(\frac{1}{\delta}+\frac{1}{2\alpha})\|\by^{(t-1)}_{I^{(t)}}-\by^{(t)}_{I^{(t)}}\|^2\\
&-2\langle \frac{1}{n}A_{I^{(t)},:}(\bar{\bx}^{(t)}-\bx^{(t)}), \by^{(t)}_{I^{(t)}}-\by^{(t-1)}_{I^{(t)}}\rangle- 2(\frac{1}{\delta}+\frac{1}{2\beta} )\|\by^{(t)}-\by^{(t-1)}\|^2
\longrnote{(from Eqn. \eqref{eqn:delta_d})}\\
    =& \La(\bx^{(t+1)},\by^{(t)})-\La(\bx^{(t)},\by^{(t)}) -2\langle \frac{1}{n}A_{I^{(t)},:}(\bar{\bx}^{(t)}-\bx^{(t)}), \by^{(t)}_{I^{(t)}}-\by^{(t-1)}_{I^{(t)}}\rangle\\
    &- (\frac{1}{\delta}+\frac{1}{\beta}-\frac{1}{2\alpha})\|\by^{(t)}-\by^{(t-1)}\|^2\\
    \leq& \La(\bx^{(t+1)},\by^{(t)})-\La(\bx^{(t)},\by^{(t)}) + 2\delta \| \frac{1}{n}A_{I^{(t)},:}(\bar{\bx}^{(t)}-\bx^{(t)})\|^2 \\
    &-(\frac{1}{\delta} -\frac{1}{2\delta})\|\by^{(t)}-\by^{(t-1)}\|^2  \shortrnote{(since $2ab\leq \gamma a^2+1/\gamma b^2$)}\\
    \leq& \La(\bx^{(t+1)},\by^{(t)})-\La(\bx^{(t)},\by^{(t)})-\frac{1}{2\delta}\|\by^{(t)}-\by^{(t-1)}\|^2\\
    &+\frac{2\delta Rk}{n^2}\|\bar{\bx}^{(t)}-\bx^{(t)}\|^2
    \end{align*}
\end{proof}
Therefore we will connect the progress induced by $-\| \by^{(t)}-\by^{(t-1)}\|$ and dual gap $\Delta_d^{(t)}$ next. 

\subsection{Dual progress}
\begin{claim}
An $\alpha$-strongly convex function $f$ satisfies:
$$ f(\bx)-f^*\leq \frac{1}{2\alpha}\|\nabla f(\bx)\|_2^2 $$
\end{claim}
This simply due to $f(\bx)-f^*\leq \langle \nabla f(\bx),\bx-\bar{\bx}\rangle-\frac{\alpha}{2}\|\bx-\bar{\bx}\|_2^2\leq \frac{1}{2\alpha}\|\nabla f(\bx)\|^2+\frac{\alpha}{2}\|\bx-\bar{\bx}\|^2-\frac{\alpha}{2}\|\bx-\bar{\bx}\|^2=\frac{1}{2\alpha}\|\nabla f(\bx)\|^2$. 

Since $-D$ is $\frac{1}{\beta}$-strongly convex, we get
\begin{align}
\nonumber 
    \Delta_d^{(t)}=D^*-D(\by^{(t)})\leq& \frac{\beta}{2} \|\nabla D(\by^{(t)})\|_2^2\\
    \nonumber 
    = &\frac{\beta}{2}\|\frac{1}{n}A\bar{\bx}^{(t)}- \bg \|^2_2\\
    \label{eqn:delta_d_bound}
    \leq& \frac{n\beta}{2k} \|\frac{1}{n}A_{\bar{I},:}\bar{\bx}^{(t)}-\bg_{\bar{I}}\|_2^2,
\end{align}
where $\bar{I}$ is a set of size $k$ that maximizes the values of $A_i^\top \bar{\bx}^{(t)}-g_i$. 

\begin{lemma}[Dual Progress]
\label{lemma:dual_progress}
\begin{equation*}
    -\|\by^{(t)} - \by^{(t-1)}\|^2
    \leq -\frac{k\delta}{n\beta}\Delta_d^{(t)} + \frac{k\delta}{n^2} R\|\bar{\bx}^{(t)}-\bx^{(t)}\|_2^2    
\end{equation*}
\end{lemma}
\begin{proof}[Proof of Lemma \ref{lemma:dual_progress}]
Define $\Delta = \frac{1}{n}A(\bar{\bx}^{(t)}-\bx^{(t)})$. Since
\begin{align*}
    & -\|\frac{1}{n}A_{I^{(t)}}^\top \bx^{(t)} - \bg_{I^{(t)}}\|^2\\
    \leq & -\|\frac{1}{n}A_{\bar{I}}^\top \bx^{(t)} - \bg_{\bar{I}}\|^2 \shortrnote{(choice of $I^{(t)}$)} \\
    = & -\|\frac{1}{n}A_{\bar{I}}^\top \bar{\bx}^{(t)} - \bg_{\bar{I}} - \Delta_{\bar{I}}\|^2\\
    \leq & -\frac{1}{2}\|\frac{1}{n}A_{\bar{I}}^\top \bar{\bx}^{(t)} - \bg_{\bar{I}} \|^2 + \|\Delta_{\bar{I}}\|_2^2   \longrnote{(since $-(a+b)^2\leq -1/2a^2+b^2$)}\\
    \leq & -\frac{k}{n\beta}\Delta_d^{(t)} + \|\Delta_{\bar{I}}\|_2^2  \shortrnote{(from \eqref{eqn:delta_d_bound} )}\\
    \leq & -\frac{k}{n\beta}\Delta_d^{(t)} + \frac{k}{n^2} R\|\bar{\bx}^{(t)}-\bx^{(t)}\|_2^2
\end{align*}
With the relation between $\frac{1}{n}A_{I^{(t)}}^\top \bx^{(t)} - \bg_{I^{(t)}}$ and $\by^{(t)} - \by^{(t-1)}$ we finish the proof.
\end{proof}
\subsection{Convergence on Duality Gap}
Now we are able to merge the primal/dual progress to get the overall progress on the duality gap.
\begin{proof}[Proof of Theorem \ref{thm:main}]
We simply blend Lemma \ref{lemma:primal_progress} and Lemma \ref{lemma:dual_progress} with the primal-dual progress (Lemma \ref{lemma:primal_dual}):
\begin{align*}
    \nonumber
    & \Delta_d^{(t)}-\Delta_d^{(t-1)}+\Delta_p^{(t)}-\Delta_p^{(t-1)}\\
    \leq & \La(\bx^{(t+1)},\by^{(t)})-\La(\bx^{(t)},\by^{(t)})-\frac{1}{2\delta}\|\by^{(t)}-\by^{(t-1)}\|^2\\
    &+\frac{2\delta Rk}{n^2}\|\bar{\bx}^{(t)}-\bx^{(t)}\|^2 \shortrnote{(Lemma \ref{lemma:primal_dual})}\\
    \leq & \La(\bx^{(t+1)},\by^{(t)})-\La(\bx^{(t)},\by^{(t)})+\frac{\delta}{2} ( -\frac{k}{n\beta}\Delta_d^{(t)} + \frac{k}{n^2} R\|\bar{\bx}^{(t)}-\bx^{(t)}\|_2^2) \\
    &+\frac{2\delta Rk}{n^2}\|\bar{\bx}^{(t)}-\bx^{(t)}\|^2 \shortrnote{(Lemma \ref{lemma:dual_progress})}\\
    =&\La(\bx^{(t+1)},\by^{(t)})-\La(\bx^{(t)},\by^{(t)}) - \frac{k\delta}{2n\beta} \Delta_d^{(t)}+ \frac{5R\delta k}{2n^2}\|\bar{\bx}^{(t)}-\bx^{(t)}\|_2^2\\
    \leq & \La(\bx^{(t+1)},\by^{(t)})-\La(\bx^{(t)},\by^{(t)}) - \frac{k\delta}{2n\beta} \Delta_d^{(t)}+ \frac{5R\delta k}{\mu n^2}(\La(\bx^{ (t)} , \by^{ (t)} ) - \La(\bar{\bx}^{ (t)} , \by^{ (t)} ))  \\
    =&(1-\frac{5R\delta k}{\mu n^2}) (\La(\bx^{(t+1)},\by^{(t)})-\La(\bx^{(t)},\by^{(t)})) - \frac{k\delta}{2n\beta} \Delta_d^{(t)} \\
    &+ \frac{5R\delta k}{\mu n^2}(\La(\bx^{(t+1)},\by^{(t)})-\La(\bar{\bx}^{(t)},\by^{(t)}) ) \\
    \leq & - \frac{k\delta}{2n\beta} \Delta_d^{(t)} - \left((1-\frac{5R\delta k}{\mu n^2})\frac{\mu}{4L} - \frac{5R\delta k}{\mu n^2} \right)\Delta_p^{(t)} \shortrnote{(Lemma \ref{lemma:primal_progress})}
\end{align*} 
When setting $\frac{k\delta}{2n\beta}=(1-\frac{5R\delta k}{\mu n^2})\frac{\mu}{4L} - \frac{5R\delta k}{\mu n^2}$, we get that 
$\Delta^{(t)}\leq \frac{1}{1+a}\Delta^{(t-1)}$, where $1/a=\Ocal(\frac{L}{\mu}(1+\frac{R\beta}{n\mu}))$. Therefore it takes $\Ocal(\frac{L}{\mu}(1+\frac{R\beta}{n\mu})\log\frac{1}{\epsilon})$ for $\Delta^{(t)}$ to reach $\epsilon$.


When $\beta>2\alpha$, we could redefine the primal-dual process as $\Delta^{(t)}:=(\frac{\beta}{\alpha}-1)\Delta^{(t)}_d+\Delta^{(t)}_p$ and 
rewrite some of the key steps, especially for the overall primal-dual progress. 

\begin{align*}
    \nonumber
    & \Delta^{(t)}-\Delta^{(t-1)}\\
    = &(\frac{\beta}{\alpha}-1)(\Delta_d^{(t)}-\Delta_d^{(t-1)})+\Delta_p^{(t)}-\Delta_p^{(t-1)}\\
    = & \La(\bx^{(t+1)},\by^{(t)})-\La(\bx^{(t)},\by^{(t)}) + \La(\bx^{(t)},\by^{(t)})-\La(\bx^{(t)},\by^{(t-1)}) \\
    & +\frac{\beta}{\alpha}(\Delta_d^{(t)}-\Delta_d^{(t-1)})\\
\leq & \La(\bx^{(t+1)},\by^{(t)})-\La(\bx^{(t)},\by^{(t)}) +(\frac{1}{\delta}+\frac{1}{2\alpha})\|\by^{(t-1)}_{I^{(t)}}-\by^{(t)}_{I^{(t)}}\|^2\\
&-\frac{\beta}{\alpha}\langle \frac{1}{n}A_{I^{(t)},:}(\bar{\bx}^{(t)}-\bx^{(t)}), \by^{(t)}_{I^{(t)}}-\by^{(t-1)}_{I^{(t)}}\rangle- \frac{\beta}{\alpha}(\frac{1}{\delta}+\frac{1}{2\beta} )\|\by^{(t)}-\by^{(t-1)}\|^2
\longrnote{(from Eqn. \eqref{eqn:update_y} and \eqref{eqn:delta_d})}\\
    =& \La(\bx^{(t+1)},\by^{(t)})-\La(\bx^{(t)},\by^{(t)}) -\frac{\beta}{\alpha}\langle \frac{1}{n}A_{I^{(t)},:}(\bar{\bx}^{(t)}-\bx^{(t)}), \by^{(t)}_{I^{(t)}}-\by^{(t-1)}_{I^{(t)}}\rangle\\
    &- (\frac{\beta}{\alpha}-1)\frac{1}{\delta}\|\by^{(t)}-\by^{(t-1)}\|^2 \\
    \leq& \La(\bx^{(t+1)},\by^{(t)})-\La(\bx^{(t)},\by^{(t)}) + \frac{3\beta}{2\alpha}\delta \| \frac{1}{n}A_{I^{(t)},:}(\bar{\bx}^{(t)}-\bx^{(t)})\|^2 \\
    &-(\frac{3\beta}{4\alpha}-1)\frac{1}{\delta}\|\by^{(t)}-\by^{(t-1)}\|^2  \shortrnote{(since $ab\leq \delta a^2+1/(4\delta) b^2$)}\\
    \leq& \La(\bx^{(t+1)},\by^{(t)})-\La(\bx^{(t)},\by^{(t)}) + \frac{\beta}{\alpha}\delta \| \frac{1}{n}A_{I^{(t)},:}(\bar{\bx}^{(t)}-\bx^{(t)})\|^2 \\
    &-\frac{\beta}{4\alpha \delta}\|\by^{(t)}-\by^{(t-1)}\|^2  \shortrnote{(since $\beta/\alpha\geq 2$)}\\
    \leq& \La(\bx^{(t+1)},\by^{(t)})-\La(\bx^{(t)},\by^{(t)})-\frac{\beta}{4\alpha\delta}\|\by^{(t)}-\by^{(t-1)}\|^2\\
    &+\frac{\beta\delta Rk}{\alpha n^2}\|\bar{\bx}^{(t)}-\bx^{(t)}\|^2
    \end{align*}
Similarly to the previous setting, we get the whole primal-dual progress is bounded as follows:
\begin{align*}
    \nonumber
    & (\frac{\beta}{\alpha}-1)(\Delta_d^{(t)}-\Delta_d^{(t-1)})+\Delta_p^{(t)}-\Delta_p^{(t-1)}\\
    \leq & \La(\bx^{(t+1)},\by^{(t)})-\La(\bx^{(t)},\by^{(t)}) - \frac{\beta\delta}{4\alpha}\frac{k}{n\beta} \Delta_d^{(t)}\\
    & + \frac{5\beta R\delta k}{2\alpha \mu n^2}(\La(\bx^{ (t)} , \by^{ (t)} ) - \La(\bar{\bx}^{ (t)} , \by^{ (t)} ))  \\
    \leq & - \frac{\beta}{4\alpha} \frac{k\delta}{n\beta} \Delta_d^{(t)} - \left((1-\frac{5\beta R\delta k}{2\alpha \mu n^2})\frac{\mu}{4L} - \frac{5\beta R\delta k}{2\alpha \mu n^2} \right)\Delta_p^{(t)} 
\end{align*} 
Therefore, when we set a proper $k$ and $\delta$ such that $ \frac{\beta}{4\alpha} \frac{k\delta}{n\beta} = (\frac{\beta}{\alpha}-1)\left( (1-\frac{5\beta R\delta k}{2\alpha \mu n^2})\frac{\mu}{4L} - \frac{5\beta R\delta k}{2\alpha \mu n^2}\right) $, and since $\frac{\beta}{\alpha}-1\geq \frac{\beta}{2\alpha}$, we get $\delta  = \frac{1}{k}(\frac{L}{\mu n \beta}+\frac{5\beta R}{2\alpha \mu n^2}(1+4\frac{L}{\mu}))^{-1} $. And we have $\Delta^{(t)} -\Delta^{(t-1)}\leq -1/a\Delta^{(t)}$, where $a =\Ocal(\frac{L}{\mu}(1+\frac{\beta}{\alpha}\frac{R\beta}{n\mu})) $. Therefore it takes $t=\Ocal(\frac{L}{\mu}(1+\frac{\beta}{\alpha}\frac{R\beta}{n\mu})\log\frac{1}{\epsilon})$ iterations for the duality gap $\Delta^{(t)}$ to reach $\epsilon$ error. 
\end{proof} 


\subsection{Smooth Hinge Loss and Relevant Properties}
\label{sec:smooth_hinge_loss}
Smooth hinge loss is defined as follows: 
\begin{equation}
    \label{def:shl}
    h(z)=\left\{\begin{array}{ll}
        \frac{1}{2}-z & \text{if }z<0 \\
        \frac{1}{2}(1-z)^2 &  \text{if } z\in [0,1] \\
        0 & \text{otherwise.}
    \end{array}\right.
\end{equation}
Our loss function over a prediction $p$ associated with a label $\ell_i \in \{\pm 1\}$ will be $f_i(p)=h(p\ell_i)$. 
The derivative of smooth hinge loss $h$ is:
\begin{equation}
    h'(z)=\left\{\begin{array}{ll}
        -1 & \text{if }z<0 \\
        z-1 & \text{if }z\in [0,1]\\
        0 & \text{otherwise.}
    \end{array}\right.
\end{equation}
Its convex conjugate is:
\begin{equation}
    h^*(z^*)=\left\{\begin{array}{ll}
    \frac{1}{2}(z^*)^2+z^* & \text{if }z^*\in [-1,0] \\
        \infty & \text{otherwise.}
    \end{array}\right.
\end{equation}
Notice since $f_i(p)=h(\ell_i p)$, $f_i^*(p) = h^*(p/\ell_i)=h^*(p\ell_i)$.
\begin{claim}
\label{claim:indicator}
For a convex and $\beta$-smooth scalar function $f$, if it is $\alpha$ strongly convex over some convex set, and linear otherwise, then its conjugate function $f^*$ is $1/\beta$-strongly convex, and it is a $1/\alpha$-smooth function plus an indicator function over some interval $[a,b]$. 
\end{claim}
\begin{proof}
To begin with, since $f''(x)\leq \beta,\forall x$, meaning $f$ is $\beta$-smooth, then with duality we have $f^*$ is $1/\beta$ strongly convex \cite{kakade2009duality}. Secondly, since $f$ is $\alpha$ strongly convex over a convex set, meaning an interval for $\R$, therefore $f$ could only be linear on $(-\infty,a]$ or $[b,\infty)$, and is $\alpha$-strongly convex over the set $[a,b]$ (Here for simplicity $a<b$ could be $\pm \infty$). We denote $f'(-\infty):=\lim_{x\rightarrow -\infty}f'(x)$ and $f'(-\infty)$ likewise. It's easy to notice that $f'(-\infty)\leq f'(a)<f'(b)\leq f'(\infty)$ since $f$ is convex overall and strongly convex over $[a,b]$. Therefore $f(y)> f(a)+f'(a)(y-a)$ when $y>a$ and $f(y)=f(a)+f'(a)(y-a)$ when $y\leq a$. 

Now since $f^*(x^*)\equiv \max_x\{x^*x-f(x) \}$, it's easy to verify that when $x^*<f'(a)$, $x^*x-f(x)=x^*x-f(a)-f'(a)(x-a)=-(f'(a)-x^*)x-f(a)+f'(a)a\rightarrow \infty$ when $x\rightarrow -\infty$. Similarly, when $x^*>f'(b)$, $f^*(x^*)=\infty$. On the other hand, when $x^*\in [f'(a),f'(b)]$, $f^*(x^*)=\max_x \{x^*x-f(x) \}=\max_{x\in [a,b]} \{x^*x-f(x)\}$. This is because $x^*a-f(a)\geq x^*y-f(y)=x^*y-f(y)-f'(a)(y-a),\forall y\leq a $, and similarly $x^*b-f(b)\geq x^*y-f(y)\forall y>b$. Therefore $f^*$ is $1/\alpha$ smooth over the interval $[f'(a),f'(b)]$, where $-\infty\leq f'(a)<f'(b)\leq \infty$. 

\end{proof}

\subsection{Convergence of Optimization over Trace Norm Ball}

The convergence analysis for trace norm ball is mostly similar to the case of $\ell_1$ ball. The most difference lies on the primal part, where our approximated update incur linear progress as well as some error. 

\begin{lemma}[Primal Progress for Algorithm \ref{alg:PD_tracenorm}]
\label{lemma:primal_progress_trace}
Suppose rank $\bar{X}^{(t)}\leq s$ and $\epsilon > 0$. If each $\tilde{X}$ computed in our algorithm is a $(\frac{1}{2},\frac{\epsilon}{8})$-approximate solution to \eqref{eqn:deltaX}, then for every $t$, it satisfies 
$\La(X^{(t+1)},Y^{(t)}) -  \La(X^{(t)},Y^{(t)})\leq -\frac{\mu}{8L} \Delta_p^{(t)}+\frac{\epsilon}{16}$. 
\end{lemma}
\begin{proof}
Refer to the proof in \cite{allen2017linear} we have:
\begin{align*}
\La(X^{(t+1)},Y^{(t)})-\La(\bar{X}^{(t)},Y^{(t)}) \leq (1-\frac{\mu}{8L}) \left(\La(X^{(t)},Y^{(t)}))-\La(\bar{X}^{(t)},Y^{(t)})\right)+\frac{\epsilon\mu}{16 L}
\end{align*}
Now move the first term on the RHS to the left and rearrange we get:
\begin{align*}
(1-\frac{\mu}{8L})(\La(X^{(t+1)},Y^{(t)})-\La(X^{(t)},Y^{(t)})) + \frac{\mu}{8L} \left(\La(X^{(t+1)},Y^{(t)}))-\La(\bar{X}^{(t)},Y^{(t)})\right) \leq\frac{\epsilon\mu}{16 L}
\end{align*}
Therefore we get:
\begin{align*}
\La(X^{(t+1)},Y^{(t)})-\La(X^{(t)},Y^{(t)})) \leq  -\frac{\mu}{8L} \Delta_p^{(t)} +\frac{\epsilon}{16}. 
\end{align*}
\end{proof}
Now back to the convergence guarantees on the trace norm ball.
\begin{proof}[Proof of Theorem \ref{thm:main_trace}]

We again define $\Delta = \frac{1}{n}A(\bar{X}^{(t)}-X^{(t)})$. $G=\nabla_Y \La(X^{(t)}, Y^{(t)})$ such that $Y^{(t)}_{I^{(t)},:}-Y^{(t-1)}_{I^{(t)},:}= \delta(\frac{1}{n}\langle A_{I^{(t)},:}X^{(t)}\rangle-G_{I^{(t)},:})$. Again we get $\|\Delta\|_F^2\leq \frac{R}{n^2}\|\bar{X}^{(t)}-X^{(t)}\|_F^2.$ 

$$ \Delta_d^{(t)}\leq \frac{\beta}{2}\|\frac{1}{n}A\bar{X}^{(t)}-G \|_F^2\leq \frac{n\beta}{2k}\|\frac{1}{n}A_{I^{(t)},:}\bar{X}^{(t)}-G_{I^{(t)},:}\|^2_F $$

Other parts are exactly the same and we get:
\begin{align*}
    \nonumber
    & (\frac{\beta}{\alpha}-1)(\Delta_d^{(t)}-\Delta_d^{(t-1)})+\Delta_p^{(t)}-\Delta_p^{(t-1)}\\
    \leq & \La(X^{(t+1)},Y^{(t)})-\La(X^{(t)},Y^{(t)}) - \frac{\beta\delta}{4\alpha}\frac{k}{n\beta} \Delta_d^{(t)}\\
    & + \frac{5\beta R\delta k}{2\alpha \mu n^2}(\La(X^{ (t)}, Y^{ (t)} ) - \La(\bar{X}^{(t)} , Y^{ (t)} ))  \\
    \leq & - \frac{\beta}{4\alpha} \frac{k\delta}{n\beta} \Delta_d^{(t)} - \left((1-\frac{5\beta R\delta k}{2\alpha \mu n^2})\frac{\mu}{8L} - \frac{5\beta R\delta k}{2\alpha \mu n^2} \right)\Delta_p^{(t)} +(1-\frac{5\beta R\delta k}{2\alpha \mu n^2})\frac{\epsilon}{16}\longrnote{(Lemma \ref{lemma:primal_progress_trace})}
\end{align*}
Therefore when $\delta \leq \frac{1}{k}(\frac{L}{\mu n \beta}+\frac{5\beta R}{2\alpha \mu n^2}(1+8\frac{L}{\mu}))^{-1} $, it satisfies $\Delta^{(t)}-\Delta^{(t-1)}\leq -\frac{k\delta}{2\beta n}\Delta^{(t)} + \frac{\epsilon}{16}.$ Therefore denote $a=\frac{2\beta n}{k\delta}$, we get $\Delta^{(t)}\leq \frac{a}{a+1}(\Delta^{(t-1)}+\frac{\epsilon}{16}) $. Therefore we get $\Delta^{(t)}\leq (\frac{a}{a+1})^t \Delta^{(0)} + \frac{\epsilon}{16} \sum_{i=1}^t (\frac{a}{a+1})^i \leq (\frac{c}{c+1})^t \Delta^{(0)} + \epsilon/16 $. Since $(\frac{a}{a+1})^t\leq e^{-t/a}$, it takes around $a=\Ocal(\frac{L}{\mu}(1+\frac{\beta}{\alpha}\frac{R\beta}{n\mu})\log\frac{1}{\epsilon})$ iterations for the duality gap to get $\epsilon$-error. 
\end{proof}

\subsection{Difficulty on Extension to Polytope Constraints} 
\label{sec:polytope}
Another important type of constraint we have not explored in this paper is the polytope constraint. Specifically, 
$$\min_{\bx \in M \subset \R^d} f(A\bx)+g(\bx) , M = conv(\Acal), \text{with only access to: LMO}_{\Acal(\br)} \in \argmin_{\bx\in \Acal} \langle \br, \bx \rangle,$$
where $\Acal \subset \R^d, |\Acal|=m$ is a finite set of vectors that is usually referred as atoms. It is worth noticing that this linear minimization oracle (LMO) for FW step naturally chooses a single vector in $\Acal$ that minimizes the inner product with $\bx$. Again, this FW step creates some "partial update" that could be appreciated in many machine learning applications. Specifically, if our computation of gradient is again dominated by a matrix-vector (data matrix versus variable $\bx$) inner product, we could possibly pre-compute each value of $\bv_i:=A\bx_i, \bx_i\in\Acal$, and simply use $\bv_i$ to update the gradient information when $\bx_i$ is the greedy direction provided by LMO. 

When connecting to our sparse update case, we are now looking for a $k$-sparse update, $k\ll m=|\Acal|$, with the basis of $\Acal$, i.e., $\tilde{\bx} =\sum_{i=1}^k \lambda_i \bx_{n_i}, \bx_{n_i}\in \Acal$. In this way, when we update $\bx^+\leftarrow (1-\eta)\bx + \eta\tilde{\bx}$, we will only need to compute $\sum_{i=1}^k \bv_{n_i}$ which is $\Ocal(kd)$ time complexity. 

However, to enforce such update that is "sparse" on $\Acal$ is much harder. To migrate our algorithms with $\ell_1$ ball or trace norm ball, we will essentially be solving the following problem:
$$ \tilde{\bx}\leftarrow \argmin_{\Lambda \in \Delta^m, \|\Lambda\|_0\leq k, \bx = \sum_{i=1}^m \lambda_i \bx_i, \bx_i \in \Acal}\langle \bg, \by \rangle + \frac{1}{2\eta}\|\by-\bx \|_2^2, $$
where $\Delta^m$ is the $m$ dimensional simplex, and $\bg$ is the current gradient vector. 

Unlike the original sparse recovery problem that could be relaxed with an $\ell_1$ constraint to softly encourage sparsity, it's generally much harder to find the $k$ sparse $\Lambda$ in this case. Actually, it is as hard as the lattice problem \cite{khot2005hardness} and is NP hard in general. 

Therefore we are not able to achieve linear convergence with cheap update with polytope-type constraints. Nonetheless, the naive FW with primal dual formulation should still be computational efficient in terms of per iteration cost, where a concentration on SVM on its dual form has been explored by \cite{lacoste2012block}. 

\section{More Results on Empirical Studies}

\subsection{More experiments with $\ell_1$ norm}
\label{sec:more_l1_result}
To investigate more on how our algorithms perform with different choices of parameters, we conducted more empirical studies with different settings of condition numbers. Specifically, we vary the parameter $\mu$ that controls the strong convexity of the primal function. Experiments are shown in Figure \ref{fig:more}.

\begin{figure*}
\begin{center}
\includegraphics[width=0.34\linewidth]{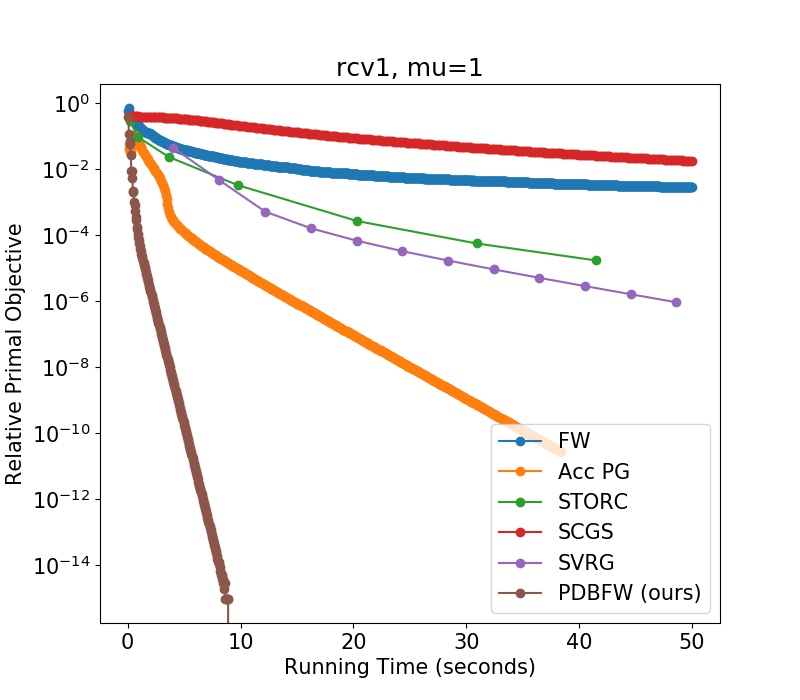}
\hspace{-15pt}
\includegraphics[width=0.34\linewidth]{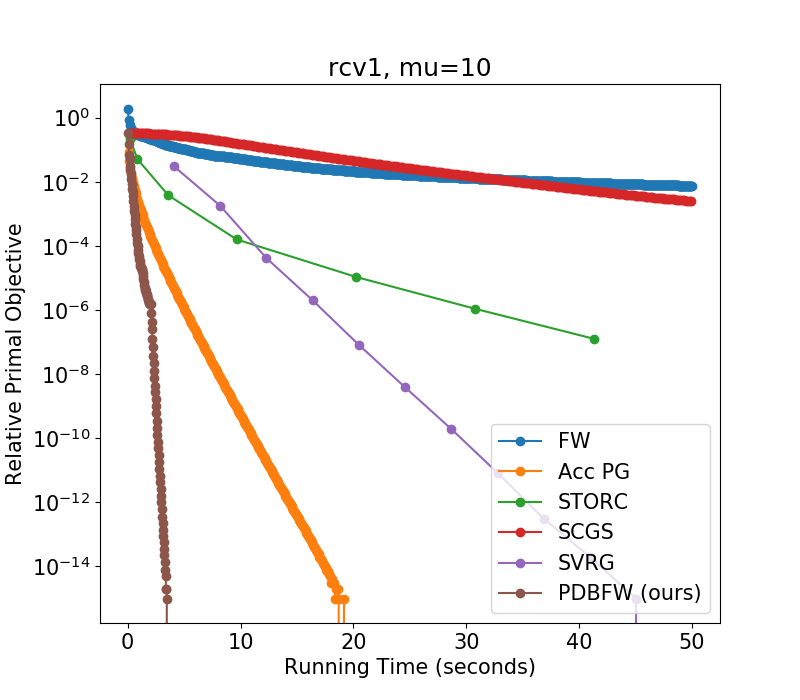}
\hspace{-15pt}
\includegraphics[width=0.34\linewidth]{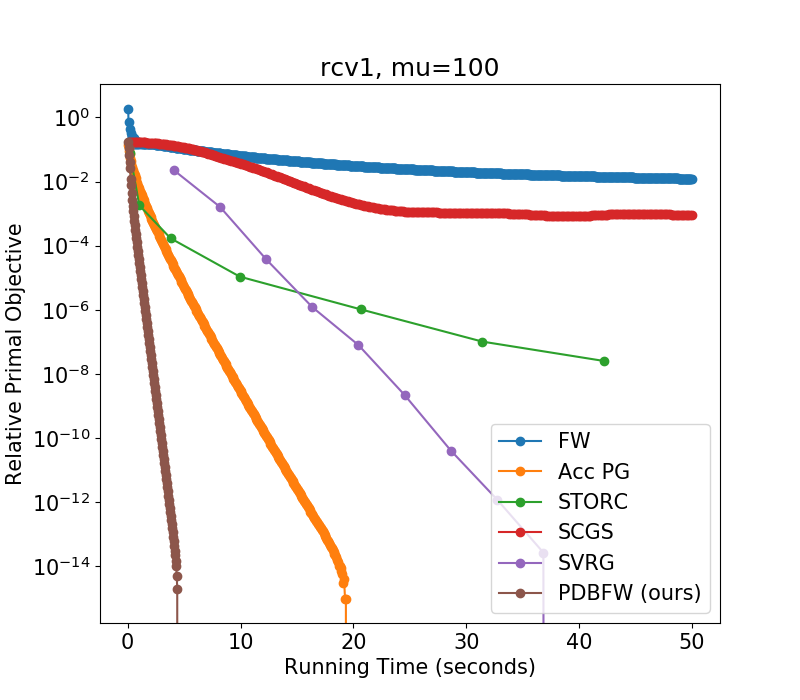}
\includegraphics[width=0.34\linewidth]{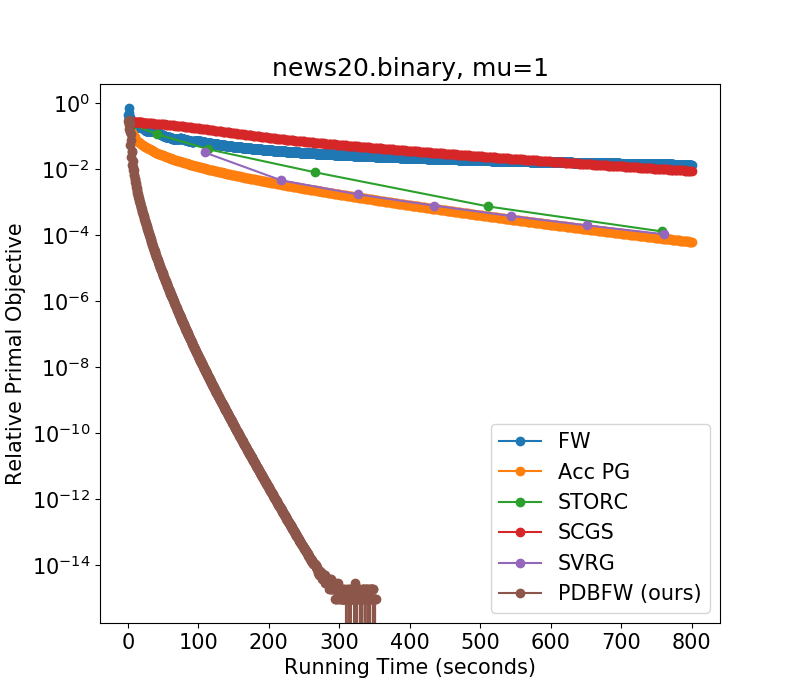}
\hspace{-15pt}
\includegraphics[width=0.34\linewidth]{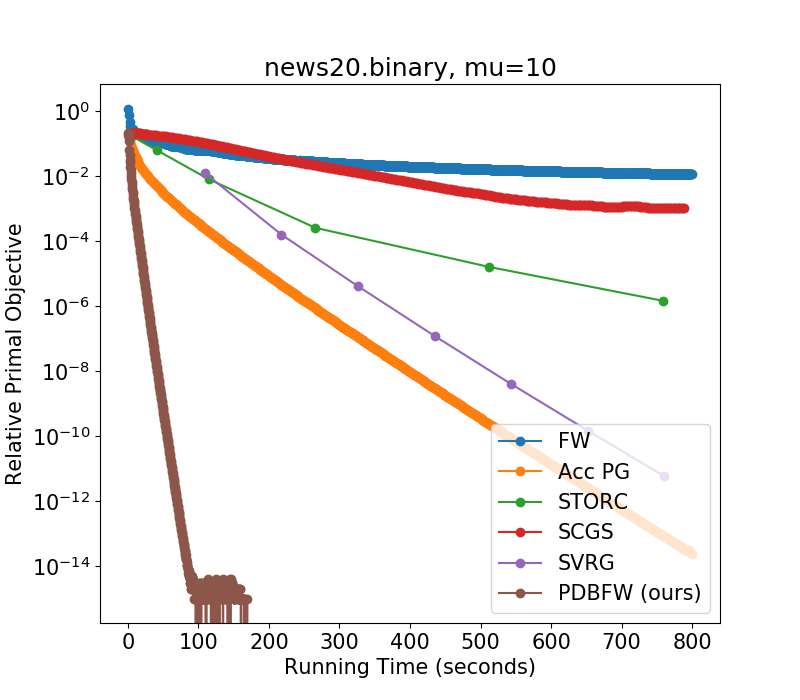}
\hspace{-15pt}
\includegraphics[width=0.34\linewidth]{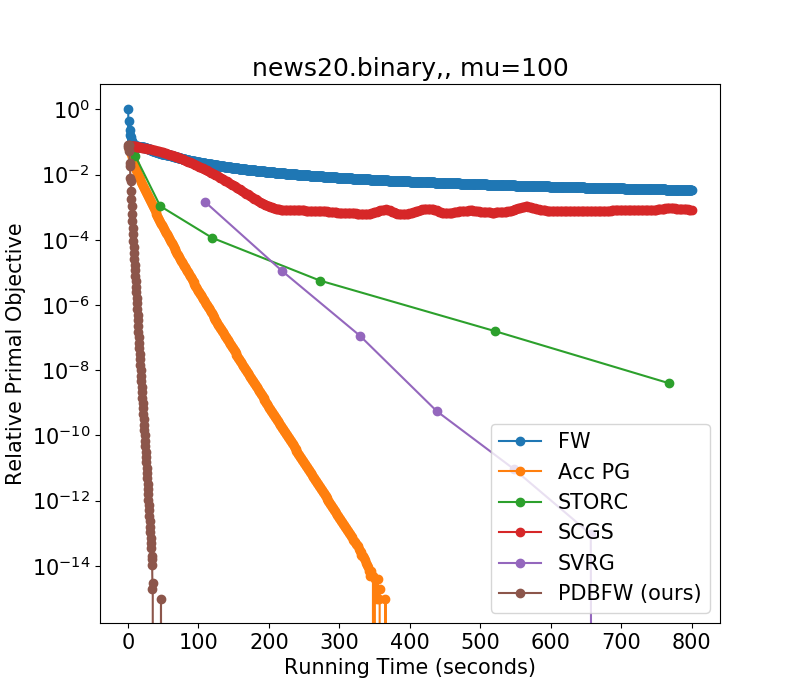}
\end{center}
\caption{Convergence result comparison of different algorithms on smoothed hinge loss by varying the coefficient of the regularizer. The first row is the results ran on the rcv1.binary dataset, while the second row is the results ran on the news20.binary dataset. The first column is the result when the regularizer coeffcient $\mu$ is set to $1/n$. The middle column is when $\mu = 10/n$, and the right column is when $\mu = 100/n$.}
\label{fig:more} 
\end{figure*}

%% file: arxiv.bbl
\begin{thebibliography}{10}

\bibitem{allen2017linear}
Zeyuan Allen-Zhu, Elad Hazan, Wei Hu, and Yuanzhi Li.
\newblock Linear convergence of a {F}rank-{W}olfe type algorithm over
  trace-norm balls.
\newblock In {\em Advances in Neural Information Processing Systems}, 2017.

\bibitem{argyriou2008convex}
Andreas Argyriou, Theodoros Evgeniou, and Massimiliano Pontil.
\newblock Convex multi-task feature learning.
\newblock {\em Machine Learning}, 2008.

\bibitem{candes2015phase}
Emmanuel~J Candes, Yonina~C Eldar, Thomas Strohmer, and Vladislav Voroninski.
\newblock Phase retrieval via matrix completion.
\newblock {\em SIAM review}, 2015.

\bibitem{chang2011libsvm}
Chih-Chung Chang and Chih-Jen Lin.
\newblock Libsvm: A library for support vector machines.
\newblock {\em ACM transactions on intelligent systems and technology (TIST)},
  2011.

\bibitem{chang2010training}
Yin-Wen Chang, Cho-Jui Hsieh, Kai-Wei Chang, Michael Ringgaard, and Chih-Jen
  Lin.
\newblock Training and testing low-degree polynomial data mappings via linear
  {SVM}.
\newblock {\em Journal of Machine Learning Research}, 2010.

\bibitem{dhillon2011nearest}
Inderjit~S Dhillon, Pradeep~K Ravikumar, and Ambuj Tewari.
\newblock Nearest neighbor based greedy coordinate descent.
\newblock In {\em Advances in Neural Information Processing Systems}, 2011.

\bibitem{du2013minimax}
Ding-Zhu Du and Panos~M Pardalos.
\newblock {\em Minimax and applications}.
\newblock Springer Science \& Business Media, 2013.

\bibitem{du2018linear}
Simon~S Du and Wei Hu.
\newblock Linear convergence of the primal-dual gradient method for
  convex-concave saddle point problems without strong convexity.
\newblock {\em arXiv preprint arXiv:1802.01504}, 2018.

\bibitem{dudik2012lifted}
Miroslav Dudik, Zaid Harchaoui, and J{\'e}r{\^o}me Malick.
\newblock Lifted coordinate descent for learning with trace-norm
  regularization.
\newblock In {\em Artificial Intelligence and Statistics}, 2012.

\bibitem{garber2015faster}
Dan Garber and Elad Hazan.
\newblock Faster rates for the {F}rank-{W}olfe method over strongly-convex
  sets.
\newblock In {\em 32nd International Conference on Machine Learning, ICML
  2015}, 2015.

\bibitem{goldfarb2017linear}
Donald Goldfarb, Garud Iyengar, and Chaoxu Zhou.
\newblock Linear convergence of stochastic {F}rank {W}olfe variants.
\newblock {\em arXiv preprint arXiv:1703.07269}, 2017.

\bibitem{eigenweb}
Ga\"{e}l Guennebaud, Beno\^{i}t Jacob, et~al.
\newblock Eigen v3.
\newblock http://eigen.tuxfamily.org, 2010.

\bibitem{hazan2016variance}
Elad Hazan and Haipeng Luo.
\newblock Variance-reduced and projection-free stochastic optimization.
\newblock In {\em International Conference on Machine Learning}, 2016.

\bibitem{jaggi2013revisiting}
Martin Jaggi.
\newblock Revisiting frank-wolfe: Projection-free sparse convex optimization.
\newblock In {\em ICML (1)}, 2013.

\bibitem{rie2013acceler}
Rie Johnson and Tong Zhang.
\newblock Accelerating stochastic gradient descent using predictive variance
  reduction.
\newblock In {\em Advances in Neural Information Processing Systems}, 2013.

\bibitem{kakade2009duality}
Sham Kakade, Shai Shalev-Shwartz, and Ambuj Tewari.
\newblock On the duality of strong convexity and strong smoothness: {L}earning
  applications and matrix regularization.

\bibitem{karimireddy2018efficient}
Sai~Praneeth Karimireddy, Anastasia Koloskova, Sebastian~U Stich, and Martin
  Jaggi.
\newblock Efficient greedy coordinate descent for composite problems.
\newblock {\em arXiv preprint arXiv:1810.06999}, 2018.

\bibitem{kerdreux2018restarting}
Thomas Kerdreux, Alexandre d'Aspremont, and Sebastian Pokutta.
\newblock Restarting {F}rank-{W}olfe.
\newblock {\em International Conference on Machine Learning}, 2019.

\bibitem{khot2005hardness}
Subhash Khot.
\newblock Hardness of approximating the shortest vector problem in lattices.
\newblock {\em Journal of the ACM (JACM)}, 2005.

\bibitem{kumar2011linearly}
Piyush Kumar and E~Alper Y{\i}ld{\i}r{\i}m.
\newblock A linearly convergent linear-time first-order algorithm for support
  vector classification with a core set result.
\newblock {\em INFORMS Journal on Computing}, 2011.

\bibitem{lacoste2015global}
Simon Lacoste-Julien and Martin Jaggi.
\newblock On the global linear convergence of {F}rank-{W}olfe optimization
  variants.
\newblock In {\em Advances in Neural Information Processing Systems}, 2015.

\bibitem{lacoste2012block}
Simon Lacoste-Julien, Martin Jaggi, Mark Schmidt, and Patrick Pletscher.
\newblock Block-coordinate {F}rank-{W}olfe optimization for structural {SVM}s.
\newblock {\em arXiv preprint arXiv:1207.4747}, 2012.

\bibitem{lan2016conditional}
Guanghui Lan and Yi~Zhou.
\newblock Conditional gradient sliding for convex optimization.
\newblock {\em SIAM Journal on Optimization}, 2016.

\bibitem{lei2017doubly}
Qi~Lei, Ian~EH Yen, Chao-yuan Wu, Inderjit~S Dhillon, and Pradeep Ravikumar.
\newblock Doubly greedy primal-dual coordinate descent for sparse empirical
  risk minimization.
\newblock In {\em Proceedings of the 34th International Conference on Machine
  Learning-Volume 70}. JMLR. org, 2017.

\bibitem{lei2016coordinate}
Qi~Lei, Kai Zhong, and Inderjit~S Dhillon.
\newblock Coordinate-wise power method.
\newblock In {\em Advances in Neural Information Processing Systems}, 2016.

\bibitem{meyer1974accelerated}
Gerard~GL Meyer.
\newblock Accelerated {F}rank-{W}olfe algorithms.
\newblock {\em SIAM Journal on Control}, 1974.

\bibitem{nanculef2014novel}
Ricardo {\~N}anculef, Emanuele Frandi, Claudio Sartori, and H{\'e}ctor Allende.
\newblock A novel {F}rank-{W}olfe algorithm. analysis and applications to
  large-scale {SVM} training.
\newblock {\em Information Sciences}, 2014.

\bibitem{nesterov2012efficiency}
Yu~Nesterov.
\newblock Efficiency of coordinate descent methods on huge-scale optimization
  problems.
\newblock {\em SIAM Journal on Optimization}, 2012.

\bibitem{nesterov2013introductory}
Yurii Nesterov.
\newblock {\em Introductory lectures on convex optimization: A basic course}.
\newblock Springer Science \& Business Media, 2013.

\bibitem{nutini2015coordinate}
Julie Nutini, Mark Schmidt, Issam Laradji, Michael Friedlander, and Hoyt
  Koepke.
\newblock Coordinate descent converges faster with the {G}auss-{S}outhwell rule
  than random selection.
\newblock In {\em International Conference on Machine Learning}, 2015.

\bibitem{pong2010trace}
Ting~Kei Pong, Paul Tseng, Shuiwang Ji, and Jieping Ye.
\newblock Trace norm regularization: Reformulations, algorithms, and multi-task
  learning.
\newblock {\em SIAM Journal on Optimization}, 2010.

\bibitem{rahimi2008random}
Ali Rahimi and Benjamin Recht.
\newblock Random features for large-scale kernel machines.
\newblock In {\em Advances in neural information processing systems}, 2008.

\bibitem{shalev2013stochastic}
Shai Shalev-Shwartz and Tong Zhang.
\newblock Stochastic dual coordinate ascent methods for regularized loss
  minimization.
\newblock {\em Journal of Machine Learning Research}, 2013.

\bibitem{shrivastava2014asymmetric}
Anshumali Shrivastava and Ping Li.
\newblock Asymmetric {LSH (ALSH)} for sublinear time maximum inner product
  search ({MIPS}).
\newblock In {\em Advances in Neural Information Processing Systems}, 2014.

\bibitem{wang2017exploiting}
Jialei Wang and Lin Xiao.
\newblock Exploiting strong convexity from data with primal-dual first-order
  algorithms.
\newblock In {\em Proceedings of the 34th International Conference on Machine
  Learning-Volume 70}. JMLR. org, 2017.

\bibitem{weintraub1985accelerating}
Andr{\'e}s Weintraub, Carmen Ortiz, and Jaime Gonz{\'a}lez.
\newblock Accelerating convergence of the {F}rank-{W}olfe algorithm.
\newblock {\em Transportation Research Part B: Methodological}, 1985.

\bibitem{xiao2019dscovr}
Lin Xiao, Adams~Wei Yu, Qihang Lin, and Weizhu Chen.
\newblock Dscovr: Randomized primal-dual block coordinate algorithms for
  asynchronous distributed optimization.
\newblock {\em Journal of Machine Learning Research}, 2019.

\bibitem{zhang2014stochastic}
Yuchen Zhang and Lin Xiao.
\newblock Stochastic primal-dual coordinate method for regularized empirical
  risk minimization.
\newblock {\em arXiv preprint arXiv:1409.3257}, 2014.

\bibitem{zou2005regularization}
Hui Zou and Trevor Hastie.
\newblock Regularization and variable selection via the elastic net.
\newblock {\em Journal of the royal statistical society: series B (statistical
  methodology)}, 2005.

\end{thebibliography}
